\documentclass[english,twoside,11pt]{article}
\usepackage[T1]{fontenc}
\usepackage[latin9]{inputenc}
\usepackage{float}
\usepackage{amsmath}
\usepackage{amsthm}
\usepackage{amssymb}
\usepackage{graphicx}
\usepackage[authoryear]{natbib}

\makeatletter

\floatstyle{ruled}
\newfloat{algorithm}{tbp}{loa}
\providecommand{\algorithmname}{Algorithm}
\floatname{algorithm}{\protect\algorithmname}

\theoremstyle{plain}
\newtheorem{thm}{\protect\theoremname}
  \theoremstyle{definition}
  \newtheorem{defn}[thm]{\protect\definitionname}
  \theoremstyle{plain}
  \newtheorem{cor}[thm]{\protect\corollaryname}
  \theoremstyle{plain}
  \newtheorem{lem}[thm]{\protect\lemmaname}
  \theoremstyle{plain}
  \newtheorem{prop}[thm]{\protect\propositionname}
  \theoremstyle{definition}
  \newtheorem{example}[thm]{\protect\examplename}
  \theoremstyle{remark}
  \newtheorem{claim}[thm]{\protect\claimname}

\usepackage{jmlr2e}
\usepackage{url}

\usepackage{lastpage}
\jmlrheading{20}{2019}{1-\pageref{LastPage}}{1/19}{6/19}{19-020}{Richard~Y. Zhang, Somayeh Sojoudi, and Javad Lavaei}
\ShortHeadings{Sharp RIP Bounds for No Spurious Local Minima}{Zhang, Sojoudi, and Lavaei}

\firstpageno{1}

\makeatother

\usepackage{babel}
  \providecommand{\claimname}{Claim}
  \providecommand{\corollaryname}{Corollary}
  \providecommand{\definitionname}{Definition}
  \providecommand{\examplename}{Example}
  \providecommand{\lemmaname}{Lemma}
  \providecommand{\propositionname}{Proposition}
\providecommand{\theoremname}{Theorem}

\begin{document}

\title{Sharp Restricted Isometry Bounds for the Inexistence of Spurious
Local Minima in Nonconvex Matrix Recovery}

\author{\name Richard Y.\ Zhang
\email ryz@illinois.edu \\        
\addr Department of Electrical and Computer Engineering\\       
University of Illinois at Urbana-Champaign\\       
306 N Wright St, Urbana, IL 61801, USA
\AND        
\name Somayeh Sojoudi
\email sojoudi@berkeley.edu \\        
\addr Department of Electrical Engineering and Computer Sciences\\       
University of California, Berkeley\\        
Berkeley, CA 94720, USA
\AND        
\name Javad Lavaei
\email lavaei@berkeley.edu \\        
\addr Department of Industrial Engineering and Operations Research\\       
University of California, Berkeley\\        
Berkeley, CA 94720, USA
} 
\editor{Sanjiv Kumar}
\maketitle
\begin{abstract}
Nonconvex matrix recovery is known to contain \emph{no spurious local
minima} under a restricted isometry property (RIP) with a sufficiently
small RIP constant $\delta$. If $\delta$ is too large, however,
then counterexamples containing spurious local minima are known to
exist. In this paper, we introduce a proof technique that is capable
of establishing \emph{sharp} thresholds on $\delta$ to guarantee
the inexistence of spurious local minima. Using the technique, we
prove that in the case of a rank-1 ground truth, an RIP constant of
$\delta<1/2$ is both necessary and sufficient for exact recovery
from any arbitrary initial point (such as a random point). We also
prove a local recovery result: given an initial point $x_{0}$ satisfying
$f(x_{0})\le(1-\delta)^{2}f(0)$, any descent algorithm that converges
to second-order optimality guarantees exact recovery.
\end{abstract}

\section{Introduction}

\global\long\def\R{\mathbb{R}}
\global\long\def\S{\mathbb{S}}
\global\long\def\A{\mathbf{A}}
\global\long\def\AA{\mathcal{A}}
\global\long\def\tr{\mathrm{tr}}
\global\long\def\vec{\mathrm{vec}\,}
\global\long\def\rank{\mathrm{rank}\,}
\global\long\def\X{\mathbf{X}}
\global\long\def\H{\mathbf{H}}
\global\long\def\HH{\mathcal{H}}
\global\long\def\e{\mathbf{e}}
\global\long\def\P{\mathbf{P}}
\global\long\def\mat{\mathrm{mat}}
\global\long\def\range{\mathrm{range}}
\global\long\def\cond{\mathrm{cond}}
\global\long\def\orth{\mathrm{orth}}
\global\long\def\ub{\mathrm{ub}}
\global\long\def\lb{\mathrm{lb}}
\global\long\def\L{\mathscr{L}}
\global\long\def\M{\mathscr{M}}
\global\long\def\LMI{\mathrm{LMI}}
The \emph{low-rank matrix recovery} problem seeks to recover an unknown
$n\times n$ ground truth matrix $M^{\star}$ of low-rank $r\ll n$
from $m$ linear measurements of $M^{\star}$. The problem naturally
arises in recommendation systems~\citep{rennie2005fast} and clustering
algorithms~\citep{amit2007uncovering}\textemdash often under the
names of matrix completion and matrix sensing\textemdash and also
finds engineering applications in phase retrieval~\citep{candes2013phaselift}
and power system state estimation~\citep{zhang2018spurious}.

In the symmetric, noiseless variant of low-rank matrix recovery, the
ground truth $M^{\star}$ is taken to be positive semidefinite (denoted
as $M^{\star}\succeq0$), and the $m$ linear measurements are made
without error, as in
\begin{equation}
b\equiv\AA(M)\quad\text{ where }\quad\AA(M)=\begin{bmatrix}\langle A_{1},M\rangle & \cdots & \langle A_{m},M\rangle\end{bmatrix}^{T}.\label{eq:defAop}
\end{equation}
To recover $M^{\star}$ from $b$, the standard approach in the machine
learning community is to factor a candidate $M$ into its low-rank
factors $XX^{T}$, and to solve a nonlinear least-squares problem
on $X$ using a local search algorithm (usually stochastic gradient
descent):
\begin{equation}
\underset{X\in\R^{n\times r}}{\text{minimize }}\quad f(X)\equiv\|\AA(XX^{T})-b\|^{2}.\label{eq:lrmr}
\end{equation}
The function $f$ is nonconvex, so a ``greedy'' local search algorithm
can become stuck at a spurious local minimum, especially if a random
initial point is used. Despite this apparent risk of failure, the
nonconvex approach remains both widely popular as well as highly effective
in practice.

Recently, \citet{bhojanapalli2016global} provided a rigorous theoretical
justification for the empirical success of local search on problem
(\ref{eq:lrmr}). Specifically, they showed that the problem contains
\emph{no spurious local minima} under the assumption that $\AA$ satisfies
the \emph{restricted isometry property} (RIP) of~\citet{recht2010guaranteed}
with a sufficiently small constant. The nonconvex problem is easily
solved using local search algorithms because every local minimum is
also a global minimum. 
\begin{defn}[Restricted Isometry Property]
The linear map $\AA:\R^{n\times n}\to\R^{m}$ is said to satisfy
$\delta$\nobreakdash-RIP if there is constant $\delta\in[0,1)$
such that 
\begin{equation}
(1-\delta)\|M\|_{F}^{2}\le\|\AA(M)\|^{2}\le(1+\delta)\|M\|_{F}^{2}\label{eq:rip}
\end{equation}
holds for all $M\in\R^{n\times n}$ satisfying $\rank(M)\le2r$.
\end{defn}
\begin{thm}[\citealp{bhojanapalli2016global,ge2017nospurious}]
\label{thm:exact_recovery}Let $\AA$ satisfy $\delta$\nobreakdash-RIP
with $\delta<1/5$. Then, (\ref{eq:lrmr}) has no spurious local minima:
\[
\nabla f(X)=0,\quad\nabla^{2}f(X)\succeq0\quad\iff\quad XX^{T}=M^{\star}.
\]
 Hence, any algorithm that converges to a second-order critical point
is guaranteed to recover $M^{\star}$ exactly.
\end{thm}
While Theorem~\ref{thm:exact_recovery} says that an RIP constant
of $\delta<1/5$ is \emph{sufficient} for exact recovery, \citet{zhang2018much}
proved that $\delta<1/2$ is \emph{necessary.} Specifically, they
gave a counterexample satisfying $1/2$-RIP that causes randomized
stochastic gradient descent to fail 12\% of the time. A number of
previous authors have attempted to close the gap between sufficiency
and necessity, including \citet{bhojanapalli2016global,ge2017nospurious,park2017non,zhang2018much,zhu2018global}.
In this paper, we prove that in the rank-1 case, an RIP constant of
$\delta<1/2$ is both \emph{necessary} and \emph{sufficient} for exact
recovery.

Once the RIP constant exceeds $\delta\ge1/2$, global guarantees are
no longer possible. \citet{zhang2018much} proved that counterexamples
exist \emph{generically}: almost every choice of $x,z\in\R^{n}$ generates
an instance of nonconvex recovery satisfying RIP with $x$ as a spurious
local minimum and $M^{\star}=zz^{T}$ as ground truth. In practice,
local search may continue to work well, often with a 100\% success
rate as if spurious local minima do not exist. However, the inexistence
of spurious local minima can no longer be assured. 

Instead, we turn our attention to local guarantees, based on good
initial guesses that often arise from domain expertise, or even chosen
randomly. Given an initial point $x_{0}$ satisfies $f(x_{0})\le(1-\delta)^{2}\|M^{\star}\|_{F}^{2}$
where $\delta$ is the RIP constant and $M^{\star}=zz^{T}$ is a rank-1
ground truth, we prove that a \emph{descent} algorithm that converges
to second-order optimality is guaranteed to recover the ground truth.
Examples of such algorithms include randomized full-batch gradient
descent \citep{jin2017escape} and trust-region methods \citep{conn2000trust,nesterov2006cubic}.

\section{Main Results}

Our main contribution in this paper is a proof technique capable of
establishing RIP thresholds that are both \emph{necessary} and \emph{sufficient}
for exact recovery. The key idea is to disprove the counterfactual.
To prove for some $\lambda\in[0,1)$ that ``$\lambda$-RIP implies
no spurious local minima'', we instead establish the inexistence
of a counterexample that admits a spurious local minimum despite satisfying
$\lambda$-RIP. In particular, if $\delta^{\star}$ is the \emph{smallest}
RIP constant associated with a counterexample, then any $\lambda<\delta^{\star}$
cannot admit a counterexample (or it would contradict the definition
of $\delta^{\star}$ as the smallest RIP constant). Accordingly, $\delta^{\star}$
is precisely the sharp threshold needed to yield a necessary and sufficient
recovery guarantee.

The main difficulty with the above line of reasoning is the need to
optimize over the set of counterexamples. Indeed, verifying RIP for
a fixed operator $\AA$ is already NP-hard in general~\citep{tillmann2014computational},
so it is reasonable to expect that optimizing over the set of RIP
operators is at least NP-hard. Surprisingly, this is not the case.
Consider finding the smallest RIP constant associated with a counterexample
with \emph{fixed} ground truth $M^{\star}=ZZ^{T}$ and \emph{fixed}
spurious point $X$:
\begin{align}
\delta(X,Z)\quad\equiv\qquad\underset{\AA}{\text{minimum}}\quad & \delta\label{eq:delta_x-1}\\
\text{subject to }\quad & f(X)=\frac{1}{2}\|\AA(XX^{T}-ZZ^{T})\|^{2}\nonumber \\
 & \nabla f(X)=0,\quad\nabla^{2}f(X)\succeq0\nonumber \\
 & \AA\text{ satisfies }\delta\text{-RIP}.\nonumber 
\end{align}
In Section~\ref{sec:Main-idea}, we reformulate problem (\ref{eq:delta_x-1})
into a \emph{convex} linear matrix inequality (LMI) optimization,
and prove that the reformulation is \emph{exact} (Theorem~\ref{thm:exact_convex}).
Accordingly, we can evaluate $\delta(X,Z)$ to arbitrary precision
in polynomial time by solving an LMI using an interior-point method. 

In the rank $r=1$ case, the LMI reformulation is sufficiently simple
that it can be relaxed and then solved in closed-form (Theorem~\ref{thm:delta_lb}).
This yields a lower-bound $\delta_{\lb}(x,z)\le\delta(x,z)$ that
we optimize over all spurious choices of $x\in\R^{n}$ to prove that
$\delta^{\star}\ge1/2$. Given that $\delta^{\star}\le1/2$ due to
the counterexample of \citet{zhang2018much}, we must actually have
$\delta^{\star}=1/2$. 
\begin{thm}[Global guarantee]
\label{thm:global}Let $r=\rank(M^{\star})=1$, let $\AA$ satisfy
$\delta$\nobreakdash-RIP, and define $f(x)=\|\AA(xx^{T}-M^{\star})\|^{2}$.
\begin{itemize}
\item If $\delta<1/2$, then $f$ has no spurious local minima:
\[
\nabla f(x)=0,\quad\nabla^{2}f(x)\succeq0\quad\iff\quad xx^{T}=M^{\star}.
\]
\item If $\delta\ge1/2,$ then there exists a counterexample $\AA^{\star}$
satisfying $\delta$-RIP, but whose $f^{\star}(x)=\|\AA^{\star}(xx^{T}-M^{\star})\|^{2}$
admits a spurious point $x\in\R^{n}$ satisfying:
\[
\|x\|^{2}=\frac{1}{2}\|M^{\star}\|_{F},\qquad f(x)=\frac{3}{4}\|M^{\star}\|_{F}^{2},\qquad\nabla f(x)=0,\qquad\nabla^{2}f(x)\succeq8\,xx^{T}.
\]
\end{itemize}
\end{thm}
We can also optimize $\delta_{\lb}(x,z)$ over spurious choices $x\in\R^{n}$
within an $\epsilon$-neighborhood of the ground truth. The resulting
guarantee is applicable to much larger RIP constants $\delta$, including
those arbitrarily close to one.
\begin{thm}[Local guarantee]
\label{thm:local}Let $r=\rank(M^{\star})=1$, and let $\AA$ satisfy
$\delta$-RIP. If 
\[
\delta<\left(1-\frac{\epsilon^{2}}{2(1-\epsilon)}\right)^{1/2}\qquad\text{where }0\le\epsilon\le\frac{\sqrt{5}-1}{2}
\]
then $f(x)=\|\AA(xx^{T}-M^{\star})\|^{2}$ has no spurious local minima
within an $\epsilon$-neighborhood of the solution:
\[
\nabla f(x)=0,\quad\nabla^{2}f(x)\succeq0,\quad\|xx^{T}-M^{\star}\|_{F}\le\epsilon\|M^{\star}\|_{F}\quad\iff\quad xx^{T}=M^{\star}.
\]
\end{thm}
Theorem~\ref{thm:local} gives an RIP-based exact recovery guarantee
for \emph{descent} algorithms, such as randomized full-batch gradient
descent~\citep{jin2017escape} and trust-region methods~\citep{conn2000trust,nesterov2006cubic},
that generate a sequence of iterates $x_{1},x_{2},\ldots,x_{k}$ from
an initial guess $x_{0}$ with each iterate no worse than the one
before:
\begin{equation}
f(x_{k})\le\cdots\le f(x_{2})\le f(x_{1})\le f(x_{0}).
\end{equation}
Heuristically, it also applies to nondescent algorithms, like stochastic
gradient descent and Nesterov's accelerated gradient descent, under
the mild assumption that the final iterate $x_{k}$ is no worse than
the initial guess $x_{0}$, as in $f(x_{k})\le f(x_{0})$. 
\begin{cor}
\label{cor:localrecov}Let $r=\rank(M^{\star})=1$, and let $\AA$
satisfy $\delta$-RIP. If $x_{0}\in\R^{n}$ satisfies
\[
f(x_{0})<(1-\delta)\epsilon^{2}f(0)\qquad\text{where }\epsilon=\min\left\{ \sqrt{1-\delta^{2}},(\sqrt{5}-1)/2\right\} ,
\]
where $f(x)=\|\AA(xx^{T}-M^{\star})\|^{2}$, then the sublevel set
defined by $x_{0}$ contains no spurious local minima:
\[
\nabla f(x)=0,\quad\nabla^{2}f(x)\succeq0,\quad f(x)\le f(x_{0})\quad\iff\quad xx^{T}=M^{\star}.
\]
\end{cor}
When the RIP constant satisfies $\delta\ge0.787$, Corollary~\ref{cor:localrecov}
guarantees exact recovery from an initial point $x_{0}$ satisfying
$f(x_{0})<(1-\delta)^{2}f(0)$. In practice, such an $x_{0}$ can
often be found using a spectral initializer~\citep{keshavan2010matrix,jain2013low,netrapalli2013phase,candes2015phase,chen2015solving}.
If $\delta$ is not too close to one, then even a random point may
suffice with a reasonable probability (see the related discussion
by~\citet{goldstein2018phasemax}). 

In the rank-$r$ case with $r>1$, our proof technique continues to
work, but $\delta(X,Z)$ becomes very challenging to solve in closed-form.
The exact RIP threshold $\delta^{\star}$ requires minimizing $\delta(X,Z)$
over \emph{all pairs} of spurious $X$ and ground truth $Z$, so the
lack of a closed-form solution would be a significant impediment to
further progress. Nevertheless, we can probe at an upper-bound on
$\delta^{\star}$ by heuristically optimizing over $X$ and $Z$,
in each case evaluating $\delta(X,Z)$ numerically using an interior-point
method. Doing this in Section~\ref{sec:Numerical-Results}, we obtain
empirical evidence that higher-rank have larger RIP thresholds, and
so are in a sense ``easier'' to solve. 

\section{Related work}

\subsection{No spurious local minima in matrix completion}

Exact recovery guarantees like Theorem~\ref{thm:exact_recovery}
have also been established for ``harder'' choices of $\AA$ that
do not satisfy RIP over its entire domain. In particular, the matrix
completion problem has sparse measurement matrices $A_{1},\ldots,A_{m}$,
with each containing just a single nonzero element. In this case,
the RIP-like condition $\|\AA(M)\|^{2}\approx\|M\|_{F}^{2}$ holds
only when $M$ is both low-rank and sufficiently dense; see the discussion
by~\citet{candes2009exact}. Nevertheless, \citet{ge2016matrix}
proved a similar result to Theorem~\ref{thm:exact_recovery} by adding
a regularizing term to the objective. 

Our recovery results are developed for the classical form of RIP\textemdash a
much stronger notion than the RIP-like condition satisfied by matrix
completion. Intuitively, if exact recovery cannot be guaranteed under
standard RIP, then exact recovery under a weaker notion would seem
unlikely. It remains future work to make this argument precise, and
to extend our proof technique to these ``harder'' choices of $\AA$.

\subsection{Noisy measurements and nonsymmetric ground truth}

Recovery guarantees for the noisy and/or nonsymmetric variants of
nonconvex matrix recovery typically require a smaller RIP constant
than the symmetric, noiseless case. For example, \citet{bhojanapalli2016global}
proved that the symmetric, zero-mean, $\sigma^{2}$-variance Gaussian
noise case requires a rank-$4r$ RIP constant of $\delta<1/10$ to
recover a $\sigma$-accurate solution $X$ satisfying $\|XX^{T}-M^{\star}\|_{F}\le20\sigma\sqrt{(\log n)/m}$.
Also, \citet{ge2017nospurious} proved that the nonsymmetric, noiseless
case requires a rank-$2r$ RIP constant $\delta<1/10$ for exact recovery.
By comparison, the symmetric, noiseless case requires only a rank-$2r$
RIP constant of $\delta<1/5$ for exact recovery.

The main goal of this paper is to develop a proof technique capable
of establishing \emph{sharp} RIP thresholds for exact recovery. As
such, we have focused our attention on the symmetric, noiseless case.
While our technique can be easily modified to accommodate for the
nonsymmetric, noisy case, the \emph{sharpness} of the technique (via
Theorem~\ref{thm:exact_convex}) may be lost. Whether an exact convex
reformulation exists for the nonsymmetric, noisy case is an open question,
and the subject of important future work. 

\subsection{Approximate second-order points and strict saddles}

Existing ``no spurious local minima'' results~\citep{bhojanapalli2016global,ge2017nospurious}
guarantee that satisfying second-order optimality to $\epsilon$-accuracy
will yield a point within an $\epsilon$-neighborhood of the solution:
\[
\|\nabla f(X)\|\le C_{1}\epsilon,\qquad\nabla^{2}f(X)\succeq-C_{2}\sqrt{\epsilon}I\qquad\iff\qquad\|XX^{T}-M^{\star}\|_{F}\le\epsilon.
\]
Such a condition is often known as ``strict saddle''~\citep{ge2015escaping}.
The associated constants $C_{1},C_{2}>0$ determine the rate at which
gradient methods can converge to an $\epsilon$-accurate solution~\citep{du2017gradient,jin2017escape}. 

The proof technique presented in this paper can be extended in a straightforward
way to the strict saddle condition. Specifically, we replace all instances
of $\nabla f(X)=0,$ $\nabla^{2}f(X)\succeq0,$ and $XX^{T}\ne M^{\star}$
with $\|\nabla f(X)\|\le C_{1}\epsilon,$ $\nabla^{2}f(X)\succeq C_{2}\sqrt{\epsilon}I,$
and $\|XX^{T}-M^{\star}\|>\epsilon$ in Section~\ref{sec:Main-idea},
and derive a suitable version of Theorem~\ref{thm:exact_convex}.
However, the resulting reformulation can no longer be solved in closed
form, so it becomes difficult to extend the guarantees in Theorem~\ref{thm:global}
and Theorem~\ref{thm:local}. Nevertheless, quantifying its asymptotic
behavior may yield valuable insights in understanding the optimization
landscape.

\subsection{Special initialization schemes}

Our local recovery result is reminiscent of classic exact recovery
results based on placing an initial point sufficiently close to the
global optimum. Most algorithms use the spectral initializer to chose
the initial point~\citep{keshavan2010matrix,keshavan2010matrixnoisy,jain2013low,netrapalli2013phase,candes2015phase,chen2015solving,zheng2015convergent,zhao2015nonconvex,bhojanapalli2016dropping,sun2016guaranteed,sanghavi2017local,park2018finding},
although other initializers have also been proposed~\citep{wang2018solving,chen2018phase,mondelli2018fundamental}.
Our result differs from prior work in being completely agnostic to
the specific application and the initializer. First, it requires only
a suboptimality bound $f(x_{0})\le(1-\delta)^{2}f(0)$ to be satisfied
by the initial point $x_{0}$. Second, its sole parameter is the RIP
constant $\delta$, so issues of sample complexity are implicitly
resolved in a universal way for different measurement ensembles. On
the other hand, the result is not directly applicable to problems
that only approximately satisfy RIP, including matrix completion.

\subsection{Comparison to convex recovery}

Classical theory for the low-rank matrix recovery problem is based
on a quadratic lift: replacing $XX^{T}$ in (\ref{eq:lrmr}) by a
convex term $M\succeq0$, and augmenting the objective with a trace
penalty $\lambda\cdot\tr(M)$ to induce a low-rank solution~\citep{candes2009exact,recht2010guaranteed,candes2010power,candes2011tight,candes2013phaselift}.
The convex approach also enjoys RIP-based exact recovery guarantees:
in the noiseless case, \citet{cai2013sharp} proved that $\delta\le1/2$
is sufficient, while the counterexample of \citet{wang2013bounds}
shows that $\delta\le1/\sqrt{2}$ is necessary. While convex recovery
may be able to solve problems with larger RIP constants than nonconvex
recovery, it is also considerably more expensive. In practice, convex
recovery is seldom used for large-scale datasets with $n$ on the
order of thousands to millions.

Recently, several authors have proposed \emph{non-lifting} convex
relaxations, motivated by the desire to avoid squaring the number
of variables in the classic quadratic lift. In particular, we mention
the PhaseMax method studied by~\citet{bahmani2017phase} and~\citet*{goldstein2018phasemax},
which avoids the need to square the number of variables when both
the measurement matrices $A_{1},\ldots,A_{m}$ and the ground truth
$M^{\star}$ are rank-1. These methods also require a good initial
guess as an input, and so are in a sense very similar to nonconvex
recovery.

\section{Preliminaries}

\subsection{Notation}

Lower-case letters are vectors and upper-case letters are matrices.
The sets $\R^{n\times n}\supset\S^{n}$ are the space of $n\times n$
real matrices and real symmetric matrices, and $\langle X,Y\rangle\equiv\tr(X^{T}Y)$
and $\|X\|_{F}^{2}\equiv\langle X,X\rangle$ are the Frobenius inner
product and norm. We write $M\succeq0$ (resp. $M\succ0$) to mean
that $M$ is positive semidefinite (resp. positive definite), and
$M\succeq S$ to denote $M-S\succeq0$ (resp. $M\succ S$ to denote
$M-S\succ0$). 

Throughout the paper, we use $X\in\R^{n\times r}$ (resp. $x\in\R^{n}$)
to refer to any candidate point, and $M^{\star}=ZZ^{T}$ (resp. $M^{\star}=zz^{T}$)
or to refer to a rank-$r$ (resp. rank-$1$) factorization of the
ground truth $M^{\star}$. The vector $\e$ and matrix $\X$ are defined
in (\ref{eq:eXdef}). We also denote the optimal value of the nonconvex
problem (\ref{eq:delta_x}) as $\delta(X,Z)$, and later show it to
be equal to the optimal value of the convex problem (\ref{eq:tau_prob})
denoted as $\LMI(X,Z)$. 

\subsection{\label{subsec:definitions}Basic definitions}

The \textbf{vectorization} operator stacks the columns of an $m\times n$
matrix $A$ into a single column vector:
\[
\vec(A)=\begin{bmatrix}A_{1,1} & \cdots & A_{m,1} & A_{1,2} & \cdots & A_{m,2} & \cdots & A_{1,n} & \cdots & A_{m,n}\end{bmatrix}^{T}.
\]
It defines an isometry between the $m\times n$ matrices $A,B$ and
their $mn$ underlying degrees of freedom $\vec(A),\vec(B)$:
\begin{align*}
\langle A,B\rangle\equiv\tr(A^{T}B) & =\vec(A)^{T}\vec(B)\equiv\langle\vec(A),\vec(B)\rangle.
\end{align*}

The \textbf{matricization} operator is the inverse of vectorization,
meaning that $A=\mat(a)$ if and only if $a=\vec(A)$. 

The \textbf{Kronecker product} between the $m\times n$ matrix $A$
and the $p\times q$ matrix $B$ is the $mp\times pq$ matrix defined
\[
A\otimes B=\begin{bmatrix}A_{1,1}B & \cdots & A_{1,n}B\\
\vdots & \ddots & \vdots\\
A_{m,1}B & \cdots & A_{m,n}B
\end{bmatrix}
\]
to satisfy the Kronecker identity 
\[
\vec(AXB^{T})=(B\otimes A)\,\vec(X).
\]

The \textbf{orthogonal basis}\emph{ }of a given $m\times n$ matrix
$A$ (with $m\ge n$) is a matrix $P=\orth(A)$ comprising $\rank(A)$
orthonormal columns of length-$m$ that span $\mathrm{range}(A)$:
\[
P=\orth(A)\qquad\iff\qquad PP^{T}A=A,\qquad P^{T}P=I_{\rank(A)}.
\]
We can compute $P$ using either a rank-revealing QR factorization~\citep{chan1987rank}
or a (thin) singular value decomposition~\citep[p. 254]{golub1989matrix}
in $O(mn^{2})$ time and $O(mn)$ memory. 

\subsection{Global optimality and local optimality}

Given a choice of $\AA:\S^{n}\to\R^{m}$ and the rank-$r$ ground
truth $M^{\star}\succeq0$, we define the nonconvex objective
\begin{equation}
f:\R^{n\times r}\to\R\qquad\text{such that}\qquad f(X)=\frac{1}{2}\|\AA(XX^{T}-M^{\star})\|^{2}.\label{eq:fdef}
\end{equation}
If the point $X$ attains $f(X)=0$, then we call it a \emph{globally
minimum}; otherwise, we call it a \emph{spurious} point. If $\AA$
satisfies $\delta$-RIP, then $X$ is a global minimum if and only
if $XX^{T}=M^{\star}$~\citep[Theorem 3.2]{recht2010guaranteed}. 

The point $X$ is said to be a \emph{local minimum} if $f(X)\le f(X')$
holds for all $X'$ within a local neighborhood of $X$. If $X$ is
a local minimum, then it must satisfy the second-order \emph{necessary}
condition for local optimality:
\begin{equation}
\nabla f(X)=0,\qquad\nabla^{2}f(X)\succeq0.\label{eq:focsoc}
\end{equation}
Conversely, a point $X$ satisfying (\ref{eq:focsoc}) is called a
\emph{second-order critical point}, and can be either a local minimum
or a saddle point. It is worth emphasizing that local search algorithms
can only guarantee convergence to a second-order critical point, and
not necessarily a local minimum; see~\citet{ge2015escaping,lee2016gradient,jin2017escape,du2017gradient}
for the literature on gradient methods, and~\citet{conn2000trust,nesterov2006cubic,cartis2012complexity,boumal2018global}
for the literature on trust-region methods. 

If a point $X$ satisfies the second-order \emph{sufficient} condition
for local optimality (with $\mu>0$):
\begin{equation}
\nabla f(X)=0,\qquad\nabla^{2}f(X)\succeq\mu I\label{eq:focsoc_suff}
\end{equation}
then it is guaranteed to be a local minimum. However, it is also possible
for $X$ to be a local minimum without satisfying (\ref{eq:focsoc_suff}).
Indeed, certifying $X$ to be a local minimum is NP-hard in the worst
case~\citep{murty1987some}. Hence, the finite gap between necessary
and sufficient conditions for local optimality reflects the inherent
hardness of the problem.

\subsection{Explicit expressions for $\nabla f(X)$ and $\nabla^{2}f(X)$}

Define $f(X)$ as the nonlinear least-squares objective shown in (\ref{eq:fdef}).
While not immediately obvious, both the gradient $\nabla f(X)$ and
the Hessian $\nabla^{2}f(X)$ are \emph{linear} with respect to the
the \emph{kernel} operator $\HH\equiv\AA^{T}\AA$. To show this, we
define the matrix representation of the operator $\AA$
\begin{equation}
\A=\begin{bmatrix}\vec(A_{1}) & \vec(A_{2}) & \cdots & \vec(A_{m})\end{bmatrix}^{T},\label{eq:Adef}
\end{equation}
which satisfies
\[
\AA(M)=\begin{bmatrix}\langle A_{1},M\rangle\\
\vdots\\
\langle A_{m},M\rangle
\end{bmatrix}=\begin{bmatrix}\vec(A_{1})^{T}\vec(M)\\
\vdots\\
\vec(A_{m})^{T}\vec(M)
\end{bmatrix}=\begin{bmatrix}\vec(A_{1})^{T}\\
\vdots\\
\vec(A_{m})^{T}
\end{bmatrix}\,\vec(M)=\A\,\vec(M).
\]
Then, some linear algebra reveals\begin{subequations}\label{eq:dfdef}
\begin{align}
f(X) & =\frac{1}{2}\e^{T}\A^{T}\A\e,\\
\nabla f(X) & =\X^{T}\A^{T}\A\e,\\
\nabla^{2}f(X) & =2\cdot[I_{r}\otimes\mat(\A^{T}\A\e)]+\X^{T}\A^{T}\A\X,
\end{align}
\end{subequations}where $\e$ and $\X$ are defined with respect
to $X$ and $M^{\star}$ to satisfy \begin{subequations}\label{eq:eXdef}
\begin{align}
\e & =\vec(XX^{T}-M^{\star}),\label{eq:edef}\\
\X\,\vec(U) & =\vec(XU^{T}+UX^{T})\qquad\forall U\in\R^{n\times r}.\label{eq:Xdef}
\end{align}
\end{subequations}(Note that $\X$ is simply the Jacobian of $\e$
with respect to $X$.) Clearly, $f(X)$, $\nabla f(X)$, and $\nabla^{2}f(X)$
are all linear with respect to $\H=\A^{T}\A$. In turn, $\H$ is simply
the matrix representation of the kernel operator $\HH$. 

As an immediate consequence noted by~\citet{zhang2018much}, both
the second-order necessary condition (\ref{eq:focsoc}) and the second-order
sufficient condition (\ref{eq:focsoc_suff}) for local optimality
are \emph{linear matrix inequalities} (LMIs) with respect to $\H$.
In particular, this means that finding an instance of (\ref{eq:lrmr})
with a fixed $M^{\star}$ as the ground truth and $X$ as a spurious
local minimum is a \emph{convex} optimization problem: 
\begin{align}
\text{find } & \AA & \text{find } & \H\succeq0\label{eq:counter2}\\
\text{such that } & f(X)=\frac{1}{2}\|\AA(XX^{T}-M^{\star})\|^{2}, & \qquad\iff\qquad\text{such that } & \X^{T}\H\e=0,\nonumber \\
 & \nabla f(X)=0, &  & 2\cdot[I_{r}\otimes\mat(\H\e)]\nonumber \\
 & \nabla^{2}f(X)\succeq\mu I. &  & \qquad+\X^{T}\H\X\succeq\mu I.\nonumber 
\end{align}
Given a feasible point $\H$, we compute an $\A$ satisfying $\H=\A^{T}\A$
using Cholesky factorization or an eigendecomposition. Then, matricizing
each row of $\A$ recovers the matrices $A_{1},\ldots,A_{m}$ implementing
a feasible choice of $\AA$.

\section{\label{sec:Main-idea}Main idea: The inexistence of counterexamples}

At the heart of this work is a simple argument by the inexistence
of counterexamples. To illustrate the idea, consider making the following
claim for a fixed choice of $\lambda\in[0,1)$ and $X,Z\in\R^{n\times r}$:
\begin{gather}
\text{If }\AA\text{ satisfies }\lambda\text{-RIP, then }X\text{ is \emph{not} a spurious second-order }\nonumber \\
\text{critical point for the nonconvex recovery of }M^{\star}=ZZ^{T}.\label{eq:nslm_claim}
\end{gather}
The claim is refuted by a counterexample: an instance of (\ref{eq:lrmr})
satisfying $\lambda$-RIP with ground truth $M^{\star}=ZZ^{T}$ and
spurious local minimum $X$. The problem of finding a counterexample
is a nonconvex feasibility problem:
\begin{align}
\text{find }\quad & \AA\label{eq:nslm_feas}\\
\text{such that }\quad & f(X)=\frac{1}{2}\|\AA(XX^{T}-ZZ^{T})\|^{2}\nonumber \\
 & \nabla f(X)=0,\quad\nabla^{2}f(X)\succeq0\nonumber \\
 & \AA\text{ satisfies }\delta\text{-RIP}.\nonumber 
\end{align}
If problem (\ref{eq:nslm_feas}) is \emph{feasible} for $\delta=\lambda$,
then any feasible point is a counterexample that refutes the claim
(\ref{eq:nslm_claim}). However, if problem (\ref{eq:nslm_feas})
is \emph{infeasible} for $\delta=\lambda$, then counterexamples do
not exist, so we must accept the claim (\ref{eq:nslm_claim}) at face
value. In other words, \emph{the inexistence of counterexamples is
proof for the original claim}. 

The same argument can be posed in an optimization form. Instead of
finding any arbitrary counterexample, we will look for the counterexample
with the \emph{smallest} RIP constant
\begin{align}
\delta(X,Z)\quad\equiv\qquad\underset{\AA}{\text{minimum}}\quad & \delta\label{eq:delta_x}\\
\text{subject to }\quad & f(X)=\frac{1}{2}\|\AA(XX^{T}-ZZ^{T})\|^{2}\nonumber \\
 & \nabla f(X)=0,\quad\nabla^{2}f(X)\succeq0\nonumber \\
 & \AA\text{ satisfies }\delta\text{-RIP}.\nonumber 
\end{align}
Suppose that problem (\ref{eq:delta_x}) attains its minimum at $\AA^{\star}$.
If $\lambda\ge\delta(X,Z)$, then the minimizer $\AA^{\star}$ is
a counterexample that refutes the claim (\ref{eq:nslm_claim}). On
the other hand, if $\lambda<\delta(X,Z)$, then problem (\ref{eq:nslm_feas})
is infeasible for $\delta=\lambda$, so counterexamples do not exist,
so the claim (\ref{eq:nslm_claim}) must be true. 

Repeating these arguments over all choices of $X$ and $Z$ yields
the following global recovery guarantee.
\begin{lem}[Sharp global guarantee]
\label{lem:delta_glob}Suppose that problem (\ref{eq:delta_x}) attains
its minimum of $\delta(X,Z)$. Define $\delta^{\star}$ as in 
\begin{equation}
\delta^{\star}\quad\equiv\quad\underset{X,Z\in\R^{n\times r}}{\text{ infimum }}\quad\delta(X,Z)\quad\text{subject to }\quad XX^{T}\ne ZZ^{T}.\label{eq:delta_glob}
\end{equation}
If $\AA$ satisfies $\lambda$-RIP with $\lambda<\delta^{\star}$,
then $f(X)=\|\AA(XX^{T}-M^{\star})\|^{2}$ with ground truth $M^{\star}\succeq0$
and $\rank(M^{\star})\le r$ satisfies:
\begin{equation}
\nabla f(X)=0,\quad\nabla^{2}f(X)\succeq0\quad\iff\quad XX^{T}=M^{\star}.\label{eq:guaran_glob}
\end{equation}
Moreover, if there exist $X^{\star},Z^{\star}$ such that $\delta^{\star}=\delta(X^{\star},Z^{\star})$,
then the threshold $\delta^{\star}$ is sharp. 
\end{lem}
\begin{proof}
To prove (\ref{eq:guaran_glob}), we simply prove the claim (\ref{eq:nslm_claim})
for $\lambda<\delta^{\star}$ and every possible choice of $X,Z\in\R^{n\times r}$.
Indeed, if $XX^{T}=ZZ^{T}$, then $X$ is not a spurious point (as
it is a global minimum), whereas if $XX^{T}\ne ZZ^{T}$, then $\lambda<\delta^{\star}\le\delta(X,Z)$
proves the inexistence of a counterexample. Sharpness follows because
the minimum $\delta^{\star}=\delta(X^{\star},Z^{\star})$ is attained
by the minimizer $\AA^{\star}$ that refutes the claim (\ref{eq:nslm_claim})
for all $\lambda\ge\delta^{\star}$ and $X=X^{\star}$ and $Z=Z^{\star}$.
\end{proof}
Repeating the same arguments over an $\epsilon$-local neighborhood
of the ground truth yields the following \emph{local} recovery guarantee.
\begin{lem}[Sharp local guarantee]
\label{lem:delta_loc}Suppose that problem (\ref{eq:delta_x}) attains
its minimum of $\delta(X,Z)$. Given $\epsilon>0$, define $\delta^{\star}(\epsilon)$
as in 
\begin{equation}
\delta^{\star}(\epsilon)\quad\equiv\quad\underset{X,Z\in\R^{n\times r}}{\text{ infimum }}\quad\delta(X,Z)\quad\text{subject to }\quad XX^{T}\ne ZZ^{T},\;\|XX^{T}-ZZ^{T}\|_{F}\le\epsilon\|ZZ^{T}\|_{F}.\label{eq:delta_x2-1}
\end{equation}
If $\AA$ satisfies $\lambda$-RIP with $\lambda<\delta^{\star}(\epsilon)$,
then $f(X)=\|\AA(XX^{T}-M^{\star})\|^{2}$ with ground truth $M^{\star}\succeq0$
and $\rank(M^{\star})\le r$ satisfies:
\begin{equation}
\nabla f(X)=0,\quad\nabla^{2}f(X)\succeq0,\quad\|XX^{T}-ZZ^{T}\|_{F}\le\epsilon\|ZZ^{T}\|_{F}\iff\quad XX^{T}=M^{\star}.\label{eq:guaran-1}
\end{equation}
Moreover, if there exist $X^{\star},Z^{\star}$ such that $\delta^{\star}=\delta(X^{\star},Z^{\star})$,
then the threshold $\delta^{\star}$ is sharp. 
\end{lem}
Our main difficulty with Lemma~\ref{lem:delta_glob} and Lemma~\ref{lem:delta_loc}
is the evaluation of $\delta(X,Z)$. Indeed, verifying $\delta$-RIP
for a fixed $\AA$ is already NP-hard in general~\citep{tillmann2014computational},
so it is reasonable to expect that solving an optimization problem
(\ref{eq:delta_x}) with a $\delta$-RIP constraint would be at least
NP-hard. Instead,~\citet{zhang2018much} suggests replacing the $\delta$-RIP
constraint with a convex \emph{sufficient} condition, obtained by
enforcing the RIP inequality (\ref{eq:rip}) over \emph{all} $n\times n$
matrices (and not just rank-$2r$ matrices):
\begin{equation}
(1-\delta)\|M\|_{F}^{2}\le\|\AA(M)\|^{2}\le(1+\delta)\|M\|_{F}^{2}\qquad\forall M\in\R^{n\times n}.\label{eq:suff}
\end{equation}
The resulting problem is a linear matrix inequality (LMI) optimization
over the kernel operator $\HH=\AA^{T}\AA$ that yields an upper-bound
on $\delta(X,Z)$:
\begin{align}
\LMI(X,Z)\quad\equiv\qquad\underset{\HH=\AA^{T}\AA}{\text{minimum}}\quad & \delta\label{eq:tau_prob}\\
\text{subject to }\quad & f(X)=\frac{1}{2}\|\AA(XX^{T}-ZZ^{T})\|^{2}\nonumber \\
 & \nabla f(X)=0,\quad\nabla^{2}f(X)\succeq0\nonumber \\
 & (1-\delta)I\preceq\AA^{T}\AA\preceq(1+\delta)I\nonumber 
\end{align}
Surprisingly, the upper-bound is \emph{tight}\textemdash problem (\ref{eq:tau_prob})
is actually an \emph{exact reformulation} of problem (\ref{eq:delta_x}). 
\begin{thm}[Exact convex reformulation]
\label{thm:exact_convex}Given $X,Z\in\R^{n\times r}$, we have $\delta(X,Z)=\LMI(X,Z)$
with both problems attaining their minima. Moreover, every minimizer
$\HH^{\star}$ for the latter problem is related to a minimizer $\AA^{\star}$
for the former problem via $\HH^{\star}=(\AA^{\star})^{T}\AA^{\star}$.
\end{thm}
Theorem~\ref{thm:exact_convex} is the key insight that allows us
to establish our main results. When rank $r=1$, the LMI is sufficiently
simple that it can be suitably relaxed and solved in closed-form,
as we will soon show in Section~\ref{sec:closed-form}. But even
when $r>1$, the LMI can still be solved numerically using an interior-point
method. This allows us to perform numerical experiments to probe at
the true value of $\delta^{\star}$ and $\delta^{\star}(\epsilon)$,
even when analytical arguments are not available. 

Section~\ref{subsec:Proof-of-tightness} below gives a proof of Theorem~\ref{thm:exact_convex}.
A key step of the proof is to establish the following equivalence:
\begin{equation}
\LMI(X,Z)=\LMI(P^{T}X,P^{T}Z)\text{ where }P=\orth([X,Z]).\label{eq:lmi_equiv}
\end{equation}
For small values of the rank $r\ll n$, equation (\ref{eq:lmi_equiv})
also yields an efficient algorithm for evaluating $\LMI(X,Z)$ in
\emph{linear} time: compute $P,$ $P^{T}X,$ and $P^{T}Z,$ and then
evaluate $\LMI(P^{T}X,P^{T}Z)$. Moreover, the associated minimizer
$\AA^{\star}$ can also be efficiently recovered. These practical
aspects are discussed in detail in Section~\ref{subsec:minimizer}.

\subsection{\label{subsec:Proof-of-tightness}Proof of Theorem~\ref{thm:exact_convex}}

Given $X,Z\in\R^{n\times r}$, we define $e\in\R^{n^{2}}$ and $\X\in\R^{n^{2}\times nr}$
to satisfy equation (\ref{eq:eXdef}) with respect to $X$ and $M^{\star}=ZZ^{T}$.
Then, problem (\ref{eq:tau_prob}) can be explicitly written as
\begin{align}
\LMI(X,Z)\quad=\quad\underset{\delta,\H}{\text{ minimum }}\quad & \delta\label{eq:delta_ub_lmi}\\
\text{subject to }\quad & \X^{T}\H\e=0,\nonumber \\
 & 2\cdot[I_{r}\otimes\mat(\H\e)]+\X^{T}\H\X\succeq0,\nonumber \\
 & (1-\delta)I\preceq\H\preceq(1+\delta)I,\nonumber 
\end{align}
with Lagrangian dual
\begin{align}
\underset{y,U_{1},U_{2},V}{\text{maximize }}\quad & \tr(U_{1}-U_{2})\label{eq:tau_dual-1}\\
\text{subject to }\quad & \tr(U_{1}+U_{2})=1,\nonumber \\
 & \sum_{j=1}^{r}(\X y-\vec(V_{j,j}))\e^{T}+\e(\X y-\vec(V_{j,j}))^{T}\nonumber \\
 & \qquad-\X V\X^{T}=U_{1}-U_{2},\nonumber \\
 & V=\begin{bmatrix}V_{1,1} & \cdots & V_{r,1}\\
\vdots & \ddots & \vdots\\
V_{r,1}^{T} & \cdots & V_{r,r}
\end{bmatrix}\succeq0,\quad U_{1}\succeq0,\quad U_{2}\succeq0.\nonumber 
\end{align}
The dual problem admits a strictly feasible point (for sufficiently
small $\epsilon>0$, set $y=0,$ $V=\epsilon I,$ $U_{1}=\eta I-\epsilon W,$
and $U_{2}=\eta\cdot I+\epsilon W$ where $2\eta=n^{-2}$ and $2W=r[\vec(I)\e^{T}+\e\vec(I)^{T}]-\X\X^{T}$)
and the primal problem is bounded (the constraints imply $\delta\ge0$).
Hence, Slater's condition is satisfied, strong duality holds, and
the primal problem attains its optimal value at a minimizer. 

It turns out that both the minimizer and the minimum are invariant
under an orthogonal projection.
\begin{lem}[Orthogonal projection]
\label{lem:reduc}Given $X,Z\in\R^{n\times r}$, let $P\in\R^{n\times q}$
with $q\le n$ satisfy
\begin{align*}
P^{T}P & =I_{q}, & PP^{T}X & =X, & PP^{T}Z & =Z.
\end{align*}
Let $(\hat{\delta},\hat{\H})$ be a minimizer for $\LMI(P^{T}X,P^{T}Z)$.
Then, $(\delta,\H)$ is a minimizer for $\LMI(X,Z)$, where $\P=P\otimes P$
and 
\begin{align*}
\delta & =\hat{\delta}, & \H & =\P\hat{\H}\P^{T}+(I-\P\P^{T}).
\end{align*}
\end{lem}
\begin{proof}
Choose arbitrarily small $\epsilon>0$. Strong duality guarantees
the existence of a dual feasible point $(\hat{y},\hat{U}_{1},\hat{U}_{2},\hat{V})$
with duality gap $\epsilon$. This is a certificate that proves $(\hat{\delta},\hat{\H})$
to be $\epsilon$-suboptimal for $\LMI(P^{T}X,P^{T}Z)$. We can mechanically
verify that $(\delta,\H)$ is primal feasible and that $(y,U_{1},U_{2},V)$
is dual feasible, where
\begin{align*}
y & =(I_{r}\otimes P)\hat{y}, & U_{1} & =\P\hat{U}_{1}\P^{T}, & U_{2} & =\P\hat{U}_{2}\P^{T}, & V & =(I_{r}\otimes P)\hat{V}(I_{r}\otimes P)^{T}.
\end{align*}
Then, $(y,U_{1},U_{2},V)$ is a certificate that proves $(\delta,\H)$
to be $\epsilon$-suboptimal for $\LMI(X,Z)$, since
\[
\delta-\tr(U_{1}-U_{2})=\hat{\delta}-\tr(\hat{U}_{1}-\hat{U}_{2})=\epsilon.
\]
Given that $\epsilon$-suboptimal certificates exist for all $\epsilon>0$,
the point $(\delta,\H)$ must actually be optimal. The details for
verifying primal and dual feasibility are straightforward but tedious;
they are included in Appendix~\ref{sec:reduc} for completeness. 
\end{proof}
Recall that we developed an upper-bound $\LMI(X,Z)$ on $\delta(X,Z)$
by replacing $\delta$-RIP with a convex \emph{sufficient} condition
(\ref{eq:suff}). The same idea can also be used to produce a lower-bound.
Specifically, we replace the $\delta$-RIP constraint with a convex
\emph{necessary} condition, obtained by enforcing the RIP inequality
(\ref{eq:rip}) over \emph{a subset of} rank-$2r$ matrices (instead
of over all rank-$2r$ matrices):
\begin{equation}
(1-\delta)\|PYP^{T}\|_{F}^{2}\le\|\AA(PYP^{T})\|^{2}\le(1+\delta)\|PYP^{T}\|_{F}^{2}\qquad\forall Y\in\R^{d\times d}\label{eq:nec_rip}
\end{equation}
where $P$ is a \emph{fixed} $n\times d$ matrix with $d\le2r$. The
resulting problem is also convex (we write $\P=P\otimes P$)
\begin{align}
\delta(X,Z)\quad\ge\quad\text{ minimize }\quad & \delta\label{eq:delta_lb_lmi}\\
\text{subject to }\quad & \X^{T}\H\e=0,\nonumber \\
 & 2\cdot[I_{r}\otimes\mat(\H\e)]+\X^{T}\H\X\succeq0,\nonumber \\
 & (1-\delta)\P^{T}\P\preceq\P^{T}\H\P\preceq(1+\delta)\P^{T}\P\nonumber 
\end{align}
with Lagrangian dual
\begin{align}
\underset{y,U_{1},U_{2},V}{\text{maximize }}\quad & \tr[\P(U_{1}-U_{2})\P^{T}]\label{eq:tau_dual-1-1}\\
\text{subject to }\quad & \tr[\P(U_{1}+U_{2})\P^{T}]=1,\nonumber \\
 & \sum_{j=1}^{r}(\X y-\vec(V_{j,j}))\e^{T}+\e(\X y-\vec(V_{j,j}))^{T}\nonumber \\
 & \qquad-\X V\X^{T}=\P(U_{1}-U_{2})\P^{T},\nonumber \\
 & V=\begin{bmatrix}V_{1,1} & \cdots & V_{r,1}\\
\vdots & \ddots & \vdots\\
V_{r,1}^{T} & \cdots & V_{r,r}
\end{bmatrix}\succeq0,\quad U_{1}\succeq0,\quad U_{2}\succeq0.\nonumber 
\end{align}
It turns out that for the specific choice of $P=\orth([X,Z])$, the
lower-bound in (\ref{eq:delta_lb_lmi}) \emph{coincides} with the
upper-bound in (\ref{eq:delta_ub_lmi}). 
\begin{lem}[Tightness]
\label{lem:tightness}Define $P=\orth([X,Z])$. Let $(\hat{\delta},\hat{\H})$
be a minimizer for $\LMI(P^{T}X,P^{T}Z)$. Then, $(\delta,\H)$ is
a minimizer for problem (\ref{eq:delta_lb_lmi}), where $\P=P\otimes P$
and 
\begin{align*}
\delta & =\hat{\delta}, & \H & =\P\hat{\H}\P^{T}.
\end{align*}
\end{lem}
\begin{proof}
The proof is almost identical to that of Lemma~\ref{lem:reduc}.
Again, choose arbitrarily small $\epsilon>0$. Let $(\hat{y},\hat{U}_{1},\hat{U}_{2},\hat{V})$
be a dual feasible point for $\LMI(P^{T}X,P^{T}Z)$ with duality gap
$\epsilon$. Then, $(y,U_{1},U_{2},V)$ where 
\begin{align*}
y & =(I_{r}\otimes P)\hat{y}, & U_{1} & =\hat{U}_{1}, & U_{2} & =\hat{U}_{2}, & V & =(I_{r}\otimes P)\hat{V}(I_{r}\otimes P)^{T}
\end{align*}
is a certificate that proves $(\delta,\H)$ to be $\epsilon$-suboptimal
for problem (\ref{eq:delta_lb_lmi}). The details for verifying primal
and dual feasibility are included in Appendix~\ref{sec:tightness}. 
\end{proof}
Putting the upper- and lower-bounds together then yields a short proof
of Theorem~\ref{thm:exact_convex}.
\begin{proof}[Proof of Theorem~\ref{thm:exact_convex}]
Denote $\delta_{\ub}=\LMI(X,Z)$ as the optimal value to the upper-bound
problem (\ref{eq:delta_ub_lmi}) and $\HH^{\star}$ as the corresponding
minimizer. (The minimizer $\HH^{\star}$ always exists due to the
boundedness of the primal problem and the existence of a strictly
feasible point in the dual problem.) Denote $\delta_{\lb}$ as the
optimal value to the lower-bound problem (\ref{eq:delta_lb_lmi}).
For $P=\orth([X,Z]),$ the sequence of inclusions 
\[
\{PYP^{T}:Y\in\R^{d\times d}\}\subseteq\{M\in\R^{n\times n}:\rank(M)\le2r\}\subseteq\R^{n\times n},
\]
implies $\delta_{\lb}\le\delta(X,Z)\le\delta_{\ub}$. However, by
Lemma~\ref{lem:reduc} and Lemma~\ref{lem:tightness}, we actually
have $\delta_{\ub}=\delta_{\lb}=\LMI(P^{T}X,P^{T}Z)$, and hence $\delta_{\lb}=\delta(X,Z)=\delta_{\ub}$.
Finally, the minimizer $\HH^{\star}$ factors into $(\AA^{\star})^{T}\AA^{\star}$,
where $\AA^{\star}$ satisfies the sufficient condition (\ref{eq:suff}),
and hence also $\delta$-RIP.
\end{proof}

\subsection{\label{subsec:minimizer}Efficient evaluation of $\delta(X,Z)$ and
$\protect\AA^{\star}$}

We now turn to the practical problem of evaluating $\delta(X,Z)$
and the associated minimizer $\AA^{\star}$ using a numerical algorithm.
While its exact reformulation $\LMI(X,Z)=\delta(X,Z)$ is indeed convex,
naïvely solving it using an interior-point solution can require up
to $O(n^{13})$ time and $O(n^{8})$ memory (as it requires solving
an order-$n^{2}$ semidefinite program). In our experiments, the largest
instances of (\ref{eq:tau_prob}) that we could accommodate using
the state-of-the-art solver MOSEK~\citep{andersen2000mosek} had
dimensions no greater than $n\le12$. 

\begin{algorithm}
\caption{\label{alg:efficient_alg}Efficient algorithm for $\delta(X,Z)$ and
$\protect\AA^{\star}$.}

\textbf{Input.} Choices of $X,Z\in\R^{n\times r}$.

\textbf{Output.} The value $\hat{\delta}=\delta(X,Z)$ and the corresponding
minimizer $\AA^{\star}$ (if desired). 

\textbf{Algorithm.} 
\begin{enumerate}
\item Compute $P=\orth([X,Z])\in\R^{n\times d}$ and project $\hat{X}=P^{T}X$
and $\hat{Z}=P^{T}Z$. 
\item Solve $\hat{\delta}=\LMI(\hat{X},\hat{Z})$ using an interior-point
method to obtain minimizer $\hat{\H}$. \textbf{Output} $\hat{\delta}$.
\item Compute the orthogonal complement $P_{\perp}=\orth(I-PP^{T})\in\R^{n\times(n-d)}.$
\item Factor $\hat{\H}=\hat{\A}^{T}\hat{\A}$ using (dense) Cholesky factorization.
\item Analytically factor $(\A^{\star})^{T}\A^{\star}=\H^{\star}=\P\hat{\H}\P^{T}+(I-\P\P^{T})$
using the formula
\[
(\A^{\star})^{T}=\begin{bmatrix}(P\otimes P)\hat{\A}^{T} & P\otimes P_{\perp} & P_{\perp}\otimes P & P_{\perp}\otimes P_{\perp}\end{bmatrix}
\]
while using the Kronecker identity $(P\otimes P)\vec(U)=\vec(PUP^{T})$
to evaluate each column of $(P\otimes P)\hat{\A}^{T}$.
\item Recover the matrices $A_{1}^{\star},\ldots,A_{m}^{\star}$ associated
with the minimizer $\AA^{\star}$ by matricizing each row of $\A^{\star}$.
\textbf{Output} $\AA^{\star}$.
\end{enumerate}
\end{algorithm}

Instead, we can efficiently evaluate $\delta(X,Z)$ using Algorithm~\ref{alg:efficient_alg}.
When the rank $r\ll n$ is small, the algorithm evaluates $\delta(X,Z)$
in linear $O(n)$ time and memory, and if desired, also recovers the
minimizer $\AA^{\star}$ in $O(n^{4})$ time and memory. In practice,
our numerical experiments were able to accommodate for rank as large
as $r\le10$.
\begin{prop}
\label{prop:alg}Algorithm~\ref{alg:efficient_alg} correctly outputs
the minimum value $\hat{\delta}=\delta(X,Z)$ and the minimizer $\AA^{\star}$.
Moreover, Steps 1-2 for $\hat{\delta}$ use
\begin{equation}
O(nr^{2}+r^{13}\log(1/\epsilon))\text{ time and }O(nr+r^{8})\text{ memory,}\label{eq:comp_val}
\end{equation}
while Steps 3-6 for $\AA^{\star}$ use
\begin{equation}
O(n^{4}+n^{2}r^{3}+nr^{4}+r^{6})\text{ time and }O(n^{4})\text{ memory.}\label{eq:comp_mini}
\end{equation}
\end{prop}
\begin{proof}
We begin by verifying correctness. The fact that the minimum value
$\delta(X,Z)=\LMI(P^{T}X,P^{T}Z)$ follows from Theorem~\ref{thm:exact_convex}
and Lemma~\ref{lem:reduc}. To prove correctness for the minimizer
$\AA^{\star}$, we recall that Algorithm~\ref{alg:efficient_alg}
defines $P_{\perp}\in\R^{n\times(n-d)}$ as the orthogonal complement
of $P\in\R^{n\times d}$, and note that
\begin{align*}
 & PP^{T}\otimes P_{\perp}P_{\perp}^{T}+P_{\perp}P_{\perp}^{T}\otimes PP^{T}+P_{\perp}P_{\perp}^{T}\otimes P_{\perp}P_{\perp}^{T}\\
= & (PP^{T}+P_{\perp}P_{\perp}^{T})\otimes(PP^{T}+P_{\perp}P_{\perp}^{T})-PP^{T}\otimes PP^{T}\\
= & I-\P\P^{T},
\end{align*}
where $\P=P\otimes P.$ Hence, Algorithm~\ref{alg:efficient_alg}
produces the minimizer $\H^{\star}=\P\hat{\H}\P^{T}+(I-\P\P^{T})$
for $\LMI(X,Z)$ in Lemma~\ref{lem:reduc} as desired. 

Now, let us quantify complexity. Note that $d\le2r=O(r)$ by construction.
Step 1 takes $O(nr^{2})$ time and $O(nr)$ memory. Step 2 requires
solving an order $\theta=O(r^{2})$ semidefinite program in $O(\theta^{6.5}\log(1/\epsilon))=O(r^{13}\log(1/\epsilon))$
time and $O(\theta^{4})=O(r^{8})$ memory. Stopping here yields (\ref{eq:comp_val}).
Step 3 uses $O(n^{3}+n^{2}r)$ time and $O(n^{2})$ memory. Step 4
uses $O(r^{6})$ time and $O(r^{4})$ memory. Step 5 performs $O(r^{2})$
matrix-vector products each costing $O(nr^{2}+n^{2}r)$ time and $O(n^{4})$
memory, and then filling the rest of $\A$ in $O(n^{4})$ time and
memory. Step 6 costs $O(n^{4})$ time and memory. Summing the terms
and substituting $O(n^{4}+r^{4})=O(n^{4})$ in the memory complexity
yields the desired figures.
\end{proof}

\section{Counterexample with $\delta=1/2$ for the rank-$1$ problem}

In this section, we use a \emph{family} of counterexamples to prove
that $\delta$-RIP with $\delta<1/2$ is necessary for the exact recovery
of \emph{any} arbitrary rank-1 ground truth $M^{\star}=zz^{T}$ (and
not just the $2\times2$ ground truth studied by \citet{zhang2018much}).
Specifically, we state a choice of $\AA^{\star}$ that satisfies $1/2$-RIP
but whose $f^{\star}(x)=\|\AA^{\star}(xx^{T}-M^{\star})\|^{2}$ admits
a spurious second-order point.
\begin{example}
\label{exa:big}Given rank-1 ground truth $M^{\star}=zz^{T}\ne0$,
define a set of orthonormal vectors $u_{1},u_{2},\ldots,u_{n}\in\R^{n}$
with $u_{1}=z/\|z\|$, and define $m=n^{2}$ measurement matrices
$A_{1},A_{2},\ldots,A_{m},$ with 
\begin{align*}
A_{1} & =u_{1}u_{1}^{T}+\frac{1}{2}u_{2}u_{2}^{T}, & A_{2} & =\frac{\sqrt{3}}{2}(u_{1}u_{2}^{T}+u_{2}u_{1}^{T}),\\
A_{n+1} & =\frac{1}{\sqrt{2}}(u_{1}u_{2}^{T}-u_{2}u_{1}^{T}), & A_{n+2} & =\frac{\sqrt{3}}{2}u_{2}u_{2}^{T},
\end{align*}
and the remaining $n^{2}-4$ measurement matrices sequentially assigned
as
\[
A_{k}=u_{i}u_{j}^{T},\quad k=i+n\cdot(j-1),\qquad\forall(i,j)\in\{1,2,\ldots,n\}^{2}\backslash\{1,2\}^{2}.
\]
Then, the associated operator $\AA^{\star}$ satisfies $1/2$-RIP:
\[
\left(1-\frac{1}{2}\right)\|M\|_{F}^{2}\le\|\AA^{\star}(M)\|^{2}\le\left(1+\frac{1}{2}\right)\|M\|_{F}^{2}\qquad\forall M\in\R^{n\times n},
\]
but the corresponding $f^{\star}(x)\equiv\|\AA^{\star}(xx^{T}-M^{\star})\|^{2}$
admits $x=(\|z\|/\sqrt{2})\,u_{2}$ as a spurious second-order critical
point:
\begin{align*}
f^{\star}(x) & =\frac{3}{4}\|M^{\star}\|_{F}^{2}, & \nabla f^{\star}(x) & =0, & \nabla^{2}f^{\star}(x) & \succeq8xx^{T}.
\end{align*}
\end{example}
We derived Example~\ref{exa:big} by numerically solving $\delta(x,z)$
with any $x$ satisfying $x^{T}z=0$ and $\|x\|=\|z\|/\sqrt{2}$ using
Algorithm~\ref{alg:efficient_alg}. The $1/2$-RIP counterexample
of \citet{zhang2018much} arises as the instance of Example~\ref{exa:big}
associated with the $2\times2$ ground truth $\hat{z}\hat{z}^{T}$
and $\hat{z}=(1,0)$:
\[
\hat{A}_{1}=\begin{bmatrix}1 & 0\\
0 & 1/2
\end{bmatrix},\quad\hat{A}_{2}=\begin{bmatrix}0 & \sqrt{3}/2\\
\sqrt{3}/2 & 0
\end{bmatrix},\quad\hat{A}_{3}=\begin{bmatrix}0 & -1/\sqrt{2}\\
1/\sqrt{2} & 0
\end{bmatrix},\quad\hat{A}_{4}=\begin{bmatrix}0 & 0\\
0 & \sqrt{3}/2
\end{bmatrix}.
\]
The associated operator $\hat{\AA}:\S^{2}\to\R^{4}$ is invertible
and satisfies $1/2$-RIP, but $\hat{x}=(0,1/\sqrt{2})$ is a spurious
second-order point:
\begin{align*}
\hat{f}(\hat{x}) & \equiv\|\hat{\AA}(\hat{x}\hat{x}^{T}-\hat{z}\hat{z}^{T})\|^{2}=\frac{3}{4}, & \nabla\hat{f}(\hat{x}) & =0, & \nabla^{2}\hat{f}(\hat{x}) & =\begin{bmatrix}0 & 0\\
0 & 4
\end{bmatrix}.
\end{align*}
We can verify the correctness of Example~\ref{exa:big} for a general
rank-1 ground truth by reducing it down to this specific $2\times2$
example. 
\begin{proof}[Proof of correctness for Example~\ref{exa:big}]
We can mechanically verify Example~\ref{exa:big} to be correct
with ground truth $\hat{z}\hat{z}^{T}$ and $\hat{z}=(1,0)$. Denote
$\hat{\AA},$ $\hat{f}(\hat{x})=\|\hat{\AA}(\hat{x}\hat{x}^{T}-\hat{z}\hat{z}^{T})\|^{2},$
and $\hat{x}=(0,1/\sqrt{2})$ as the corresponding minimizer, nonconvex
objective, and spurious second-order critical point. 

For a general rank-1 ground truth $M^{\star}=zz^{T}$, recall that
we have defined a set of orthonormal vectors $u_{1},u_{2},\ldots,u_{n}\in\R^{n}$
with $u_{1}=z/\|z\|$. Then, setting $P=[u_{1},u_{2}]$ and $P_{\perp}=[u_{3},\ldots,u_{n}]$
shows that the matrix version of $\AA^{\star}$ can be permuted row-wise
to satisfy
\[
(\A^{\star})^{T}=\begin{bmatrix}(P\otimes P)\hat{\A}^{T} & P\otimes P_{\perp} & P_{\perp}\otimes P & P_{\perp}\otimes P_{\perp}\end{bmatrix}
\]
where $\hat{\A}$ is the matrix version of $\hat{\AA}$. Repeating
the proof of Proposition~\ref{prop:alg} shows that
\[
(\A^{\star})^{T}\A^{\star}=\P\hat{\A}^{T}\hat{\A}\P+(I-\P\P^{T})
\]
where $\P=P\otimes P$, and so $\AA^{\star}$ also satisfies $1/2$-RIP.
Moreover, this implies that
\[
f^{\star}(x)\equiv\|\AA^{\star}(xx^{T}-zz^{T})\|^{2}=\|z\|^{4}\hat{f}(P^{T}x/\|z\|)+(\|x\|^{4}-\|P^{T}x\|^{4}).
\]
Differentiating yields the following at $x=(\|z\|/\sqrt{2})u_{2}$:
\begin{alignat*}{2}
f^{\star}(x) & =\|z\|^{4}\hat{f}(\hat{x}) &  & =(3/4)\|z\|^{4},\\
\nabla f^{\star}(x) & =\|z\|^{3}P\nabla\hat{f}(\hat{x}) &  & =0,\\
\nabla^{2}f^{\star}(x) & =\|z\|^{2}P\nabla^{2}\hat{f}(\hat{x})P^{T}+2\|x\|^{2}(I-PP^{T}) & \quad & \succeq4\|z\|^{2}u_{2}u_{2}^{T}.
\end{alignat*}
\end{proof}

\section{\label{sec:closed-form}Closed-form lower-bound for the rank-$1$
problem}

It turns out that the LMI problem (\ref{eq:tau_prob}) in the rank-1
case is sufficiently simple to be suitably relaxed and then solved
in closed-form. Our main result in this section is the following lower-bound
on $\delta(x,z)=\LMI(x,z)$. 
\begin{thm}[Closed-form lower-bound]
\label{thm:delta_lb}Let $x,z\in\R^{n}$ be arbitrary nonzero vectors,
and define their length ratio $\rho$ and incidence angle $\phi$:
\begin{align}
\rho & \equiv\frac{\|x\|}{\|z\|}, & \phi & \equiv\arccos\left(\frac{x^{T}z}{\|x\|\|z\|}\right).\label{eq:rho_phi_def}
\end{align}
Define the following two scalars with respect to $\rho$ and $\phi$:
\begin{align*}
\alpha & =\frac{\sin^{2}\phi}{\sqrt{(\rho^{2}-1)^{2}+2\rho^{2}\sin^{2}\phi}}, & \beta & =\frac{\rho^{2}}{\sqrt{(\rho^{2}-1)^{2}+2\rho^{2}\sin^{2}\phi}}.
\end{align*}
Then, we have $\delta(x,z)\ge\delta_{\lb}(x,z)$, where
\begin{alignat}{2}
\delta_{\lb}(x,z)\quad\equiv\quad & \sqrt{1-\alpha^{2}} & \qquad\text{if }\beta & \ge\frac{\alpha}{1+\sqrt{1-\alpha^{2}}},\label{eq:delta_lb_a}\\
 & \frac{1-2\alpha\beta+\beta^{2}}{1-\beta^{2}} & \text{if }\beta & \le\frac{\alpha}{1+\sqrt{1-\alpha^{2}}}.\label{eq:delta_lb_b}
\end{alignat}
\end{thm}
The rank-1 global and local recovery guarantees follow quickly from
this theorem, as shown below. 
\begin{proof}[Proof of Theorem~\ref{thm:global}]
The existence of Example~\ref{exa:big} already proves that 
\begin{equation}
\delta^{\star}=\min_{x,z\in\R^{n}}\delta(x,z)\le1/2.\label{eq:glob_min_deltaub}
\end{equation}
Below, we will show that $\delta_{\lb}(x,z)$ attains its minimum
of $1/2$ at any $x$ satisfying $x^{T}z=0$ and $\|x\|/\|z\|=1/\sqrt{2}$,
as in
\begin{equation}
1/2=\min_{x,z\in\R^{n}}\delta_{\lb}(x,z)\le\delta^{\star}.\label{eq:glob_min_deltalb}
\end{equation}
Substituting $\delta^{\star}=1/2$ into Lemma~\ref{lem:delta_glob}
then completes the proof of our global recovery guarantee in Theorem~\ref{thm:global}. 

To prove (\ref{eq:glob_min_deltalb}), we begin by optimizing $\delta_{\lb}(x,z)$
over the region $\beta\ge\alpha/(1+\sqrt{1-\alpha^{2}})$ using equation
(\ref{eq:delta_lb_a}), and find that the minimum value is attained
along the boundary 
\[
\beta=\frac{\alpha}{1+\sqrt{1-\alpha^{2}}}=\frac{1-\sqrt{1-\alpha^{2}}}{\alpha}.
\]
Note that the two equations (\ref{eq:delta_lb_a}) and (\ref{eq:delta_lb_b})
coincide at this boundary:
\begin{gather*}
\left(\frac{1-2\alpha\beta+\beta^{2}}{1-\beta^{2}}\right)\left(\frac{\alpha/\beta}{\alpha/\beta}\right)=\frac{(1+\sqrt{1-\alpha^{2}})-2\alpha^{2}+(1-\sqrt{1-\alpha^{2}})}{(1+\sqrt{1-\alpha^{2}})-(1-\sqrt{1-\alpha^{2}})}=\sqrt{1-\alpha^{2}}.
\end{gather*}
Now, we optimize $\delta_{\lb}(x,z)$ over the region $\beta\le\alpha/(1+\sqrt{1-\alpha^{2}})$
using equation (\ref{eq:delta_lb_b}). First, substituting the definitions
of $\alpha$ and $\beta$ yields
\[
\delta_{\lb}(x,z)={\displaystyle \frac{1-2\alpha\beta+\beta^{2}}{1-\beta^{2}}}=\frac{(\rho^{4}+1-2\rho^{2}\cos^{2}\phi)-2\rho^{2}\sin^{2}\phi+\rho^{4}}{(\rho^{4}+1-2\rho^{2}\cos^{2}\phi)-\rho^{4}}=\frac{(\rho^{2}-1)^{2}+\rho^{4}}{1-2\rho^{2}\cos^{2}\phi}.
\]
This expression is minimized at $\phi=\pm\pi/2$ and $\rho=1/\sqrt{2}$,
with a minimum value of $1/2$. The corresponding point $\alpha=2/\sqrt{5}$
and $\beta=\alpha/4$ lies in the strict interior $\beta<\alpha/(1+\sqrt{1-\alpha^{2}})$.
This point must be the global minimum, because it dominates the boundary
$\beta=\alpha/(1+\sqrt{1-\alpha^{2}})$, which in turn dominates the
other region $\beta>\alpha/(1+\sqrt{1-\alpha^{2}})$. 
\end{proof}
\begin{proof}[Proof of Theorem~\ref{thm:local}]
We will optimize over an $\epsilon$-neighborhood of the ground truth
and show that
\begin{equation}
\left(1-\frac{\epsilon^{2}}{2(1-\epsilon)}\right)^{1/2}\le\min_{x,z\in\R^{n}}\{\delta_{\lb}(x,z):\|xx^{T}-zz^{T}\|_{F}\le\epsilon\|zz^{T}\|_{F}\}\le\delta^{\star}(\epsilon).\label{eq:glob_min_deltalb-1}
\end{equation}
Substituting this lower-bound on $\delta^{\star}(\epsilon)$ into
Lemma~\ref{lem:delta_loc} then completes the proof of our local
recovery guarantee in Theorem~\ref{thm:local}. 

To obtain (\ref{eq:glob_min_deltalb-1}), we first note that the $\epsilon$-neighborhood
constraint implies the following 
\[
\|xx^{T}-zz^{T}\|_{F}\le\epsilon\|zz^{T}\|_{F}\qquad\iff\qquad(\rho^{2}-1)^{2}+2\rho^{2}\sin^{2}\phi\le\epsilon^{2}.
\]
This in turn implies $\epsilon^{2}\ge(\rho^{2}-1)^{2}$ and $\epsilon^{2}\ge[(\rho^{2}-1)^{2}+2\rho^{2}]\sin^{2}\phi,$
and hence
\[
1-\epsilon\le\rho^{2}\le1+\epsilon,\qquad\sin^{2}\phi\le\epsilon^{2}.
\]
We wish to derive a threshold $\hat{\epsilon}$ such that if $\epsilon\le\hat{\epsilon}$,
then
\[
\frac{\beta}{\alpha}=\frac{\rho^{2}}{\sin^{2}\phi}\ge\frac{1-\epsilon}{\epsilon^{2}}\ge1\ge\frac{1}{1+\sqrt{1-\alpha^{2}}},
\]
and so $\delta_{\lb}(x,z)=\sqrt{1-\alpha^{2}}$ as dictated entirely
by equation (\ref{eq:delta_lb_a}). Clearly, this requires solving
the quadratic equation $(1-\hat{\epsilon})=\hat{\epsilon}^{2}$ for
the positive root at $\hat{\epsilon}=(-1+\sqrt{5})/2\ge0.618$. Now,
we upper-bound $\alpha^{2}$ to lower-bound $\sqrt{1-\alpha^{2}}$:
\begin{align*}
\alpha^{2} & =\frac{\sin^{4}\phi}{(\rho^{2}-1)^{2}+2\rho^{2}\sin^{2}\phi}\le\frac{\sin^{4}\phi}{[(\rho^{2}-1)^{2}+2\rho^{2}]\sin^{2}\phi}=\frac{\sin^{2}\phi}{\rho^{4}+1}\\
 & \le\frac{\epsilon^{2}}{(1-\epsilon)^{2}+1}=\frac{\epsilon^{2}}{2-2\epsilon+\epsilon^{2}}\le\frac{\epsilon^{2}}{2(1-\epsilon)}
\end{align*}
and so
\[
\delta(x,z)\ge\delta_{\lb}(x,z)=\sqrt{1-\alpha^{2}}\ge\sqrt{1-\frac{\epsilon^{2}}{2(1-\epsilon)}}.
\]
\end{proof}
\begin{proof}[Proof of Corollary~\ref{cor:localrecov}]
Under $\delta$-RIP, a point with a small residual must also have
a small error:
\begin{gather}
(1-\delta)\|xx^{T}-M^{\star}\|_{F}^{2}\le f(x)\le(1-\delta)\epsilon^{2}\|M^{\star}\|_{F}^{2}.\label{eq:resid_bnd}
\end{gather}
In particular, any point in the level set $f(x)\le f(x_{0})$ must
also lie in the $\epsilon$-neighborhood:
\[
f(x)\le f(x_{0})<(1-\delta)\epsilon^{2}f(0)\qquad\implies\qquad\|xx^{T}-M^{\star}\|_{F}\le\epsilon\|M^{\star}\|_{F}.
\]
Additionally, note that 
\[
\epsilon^{2}\le1-\delta^{2},\quad\epsilon^{2}\le\frac{\sqrt{5}-1}{2}\qquad\implies\qquad\delta\le\sqrt{1-\frac{\epsilon^{2}}{2(1-\epsilon)}}
\]
because $2(1-\epsilon)\le1$. The result then follows by applying
Theorem~\ref{thm:local}.
\end{proof}
The rest of this section is devoted to proving Theorem~\ref{thm:delta_lb}.
We begin by providing a few important lemmas in Section~\ref{subsec:Technical-lemmas},
and then move to the proof itself in Section~\ref{subsec:proof_delta_lb}. 

\subsection{\label{subsec:Technical-lemmas}Technical lemmas}

Given $M\in\S^{n}$ with eigendecomposition $M=\sum_{i=1}^{m}\lambda_{i}v_{i}v_{i}^{T}$,
we define its projection onto the semidefinite cone as the following
\begin{align*}
[M]_{+} & \equiv\arg\min_{S\succeq0}\|M-S\|_{F}^{2}=\sum_{i=1}^{n}\max\{\lambda_{i},0\}v_{i}v_{i}^{T}.
\end{align*}
For notational convenience, we also define a complement projection
\[
[M]_{-}\equiv[-M]_{+}=[M]_{+}-M,
\]
thereby allowing us to decompose every $M$ into a positive and a
negative part as in 
\[
M=[M]_{+}-[M]_{-}\quad\text{where }\quad[M]_{+}\succeq0,\quad[M]_{-}\succeq0.
\]

\begin{lem}
\label{lem:dual_prob}Given $M\in\S^{n}$ with $\tr(M)\ge0$, the
following problem
\[
\underset{\alpha,U,V}{\text{{\rm minimize}}}\quad\tr(V)\quad\text{ {\rm subject to }}\quad\tr(U)=1,\quad\alpha M=U-V,\quad U,V\succeq0
\]
has minimizer 
\begin{align*}
\alpha^{\star} & =1/\tr([M]_{+}), & U^{\star} & =\alpha^{\star}\cdot[M]_{+}, & V^{\star} & =\alpha^{\star}\cdot[M]_{-}.
\end{align*}
\end{lem}
\begin{proof}
Write $p^{\star}$ as the optimal value. Then, 
\begin{align*}
p^{\star}= & \max_{\beta}\min_{\begin{subarray}{c}
\alpha\in\R\\
U,V\succeq0
\end{subarray}}\{\tr(V)+\beta\cdot[\tr(U)-1]:\alpha M=U-V\}\\
= & \max_{\beta\ge0}\min_{\alpha\in\R}\{-\beta+\min_{U,V\succeq0}\{\tr(V)+\beta\cdot\tr(U):\alpha M=U-V\}\}\\
= & \max_{\beta\ge0}\min_{\alpha\in\R}\{-\beta+\alpha\cdot[\tr([M]_{-})+\beta\cdot\tr([M]_{+})]\}\\
= & \max_{\beta\ge0}\{-\beta:\tr([M]_{-})+\beta\cdot\tr([M]_{+})=0\}\\
= & \tr([M]_{-})/\tr([M]_{+})=\tr(V^{\star}).
\end{align*}
The first line converts an equality constraint into a Lagrangian.
The second line isolates the optimization over $U,V\succeq0$ with
$\beta\ge0$, noting that $\beta<0$ would yield $\tr(U)\to\infty$.
The third line solves the minimization over $U,V\succeq0$ in closed-form.
The fourth line views $\alpha$ as a Lagrange multiplier. 
\end{proof}
For symmetric indefinite matrices of a particular rank-2 form, the
positive and negative eigenvalues can be computed in closed-form.
\begin{lem}
\label{lem:eig_bound}Given $a,b\in\R^{n}$, the matrix $M=ab^{T}+ba^{T}$
has eigenvalues $\lambda_{1}\ge\cdots\ge\lambda_{n}$ where:
\[
\lambda_{i}=\begin{cases}
+\|a\|\|b\|(1+\cos\theta) & i=1\\
-\|a\|\|b\|(1-\cos\theta) & i=n\\
0 & \text{otherwise,}
\end{cases}
\]
and $\theta\equiv\arccos\left(\frac{a^{T}b}{\|a\|\|b\|}\right)$ is
the angle between $a$ and $b$. 
\end{lem}
\begin{proof}
Without loss of generality, assume that $\|a\|=\|b\|=1$. (Otherwise,
rescale $\hat{a}=a/\|a\|,$ $\hat{b}=b/\|b\|$ and write $M=\|a\|\|b\|\cdot(\hat{a}\hat{b}^{T}+\hat{b}\hat{a}^{T})$.)
Decompose $b$ into a tangent and normal component with respect to
$a$, as in
\[
b=a\underbrace{a^{T}b}_{\cos\theta}+\underbrace{(I-aa^{T})b}_{c\sin\theta}=a\cos\theta+c\sin\theta,
\]
where $c$ is a unit normal vector with $\|c\|=1$ and $a^{T}c=0$.
This allows us to write
\[
ab^{T}+ba^{T}=\begin{bmatrix}a & c\end{bmatrix}\begin{bmatrix}2\cos\theta & \sin\theta\\
\sin\theta & 0
\end{bmatrix}\begin{bmatrix}a & c\end{bmatrix}^{T}
\]
and hence $M$ is spectrally similar a $2\times2$ matrix with eigenvalues
$\cos\theta\pm\sqrt{\cos^{2}\theta+\sin^{2}\theta}$. 
\end{proof}
Given $x,z\in\R$, recall that $\e$ and $\X$ are implicitly defined
in (\ref{eq:eXdef}) to satisfy
\begin{align*}
\e & =\vec(xx^{T}-zz^{T}), & \X y & =\vec(xy^{T}+yx^{T})\qquad\forall y\in\R^{n}.
\end{align*}
Let us give a preferred orthogonal basis to study these two objects.
We define $v_{1}=x/\|x\|$ in the direction of $x$. Then, we decompose
$z$ into a tangent and normal component with respect to $v_{1}$,
as in
\begin{equation}
z=v_{1}\underbrace{v_{1}^{T}z}_{\|z\|\cos\phi}+\underbrace{(I-v_{1}v_{1}^{T})z}_{v_{2}\|z\|\sin\phi}=\|z\|\cdot(v_{1}\cos\phi+v_{2}\sin\phi).\label{eq:zdecomp}
\end{equation}
Here, $\phi$ is the incidence angle between $x$ and $z$, and $v_{2}$
is the associated unit normal vector with $\|v_{2}\|=1$ and $v_{1}^{T}v_{2}=0$.
Using the Gram-Schmidt process, we complete $v_{1},v_{2}$ with the
remaining $n-2$ set of orthonormal unit vectors $v_{3},v_{4},\dots,v_{n}$.
This results in a set of right singular vectors for $\X$.
\begin{lem}
\label{lem:svd}The matrix $\X\in\R^{n^{2}\times n}$ has singular
value decomposition $\X=\sum_{i=1}^{n}\sigma_{i}u_{i}v_{i}^{T}$ where
$v_{i}$ are defined as above, and 
\[
\sigma_{i}=\begin{cases}
2\|x\| & i=1\\
\sqrt{2}\|x\| & i>1
\end{cases},\qquad u_{i}=\begin{cases}
v_{1}\otimes v_{1} & i=1\\
\frac{1}{\sqrt{2}}(v_{i}\otimes v_{1}+v_{1}\otimes v_{i}) & i>1
\end{cases}.
\]
\end{lem}
\begin{proof}
It is easy to verify that 
\begin{align*}
\X y & =y\otimes x+x\otimes y=\|x\|(y\otimes v_{1}+v_{1}\otimes y)\\
 & =\|x\|\cdot\sum_{i=1}^{n}(v_{i}^{T}y)(v_{i}\otimes v_{1}+v_{1}\otimes v_{i}).
\end{align*}
Normalizing the left singular vectors then yields the designed $u_{i}$
and $\sigma_{i}$.
\end{proof}
We can also decompose $\e$ into a tangent and normal component with
respect to $\mathrm{range}(\X)$ as in
\begin{equation}
\e=\underbrace{\X\X^{\dagger}\e}_{\hat{e}_{1}\|\e\|\cos\theta}+\underbrace{(I-\X\X^{\dagger})\e}_{\hat{e}_{2}\|\e\|\sin\theta}=\|\e\|\cdot(\hat{e}_{1}\cos\theta+\hat{e}_{2}\sin\theta)\label{eq:edecomp}
\end{equation}
where $\X^{\dagger}=(\X^{T}\X)^{-1}\X^{T}$ is the usual pseudoinverse.
The following Lemma gives the exact values of $\hat{e}_{2}$ and $\sin\theta$
(thereby also implicitly giving $\hat{e}_{1}$ and $\cos\theta$).
\begin{lem}
\label{lem:edecomp}Define $\phi$ and $v_{2}$ as in (\ref{eq:zdecomp}),
we have 
\[
(I-\X\X^{\dagger})\e=-(v_{2}\otimes v_{2})(\|z\|\sin\phi)^{2}
\]
and hence $\hat{e}_{2}=v_{2}\otimes v_{2}$ and $\sin\theta=(\|z\|\sin\phi)^{2}/\|\e\|$.
\end{lem}
\begin{proof}
We solve the projection problem
\begin{align*}
\|(I-\X\X^{\dagger})\e\| & =\min_{y}\|\e-\X y\|=\|(xx^{T}-zz^{T})-(xy^{T}+yx^{T})\|_{F}\\
 & =\min_{\alpha,\beta}\left\Vert \begin{bmatrix}\|x\|^{2}-\|z\|^{2}\cos^{2}\phi & -\|z\|^{2}\sin\phi\cos\phi\\
-\|z\|^{2}\sin\phi\cos\phi & -\|z\|^{2}\sin^{2}\phi
\end{bmatrix}-\begin{bmatrix}2\alpha & \beta\\
\beta & 0
\end{bmatrix}\right\Vert \\
 & =\|z\|^{2}\sin^{2}\phi
\end{align*}
in which the second line makes a change of bases to $v_{1}$ and $v_{2}$.
Clearly, the minimizer is in the direction of $-v_{2}\otimes v_{2}$.
\end{proof}
Using these properties of $\X$ and $\e$, we can now solve the following
problem in closed-form. 
\begin{lem}
\label{lem:keylb}Define $\alpha=(\|z\|\sin\phi)^{2}/\|e\|=\sin\theta$
and $\beta=\|x\|^{2}/\|e\|$. Then, the following optimization problem
\begin{align*}
\psi(\gamma)\quad\equiv\quad\underset{y,W}{\text{{\rm maximum}}}\quad & \e^{T}[\X y-\vec(W)]\\
\text{{\rm subject to}}\quad & \|\e\|\cdot\|\X y-\vec(W)\|=1\\
 & \tr(\X W\X^{T})=2\beta\cdot\gamma\\
 & W\succeq0
\end{align*}
is feasible if and only if $0\le\gamma\le1$ with optimal value
\[
\psi(\gamma)=\gamma\alpha+\sqrt{1-\gamma^{2}}\sqrt{1-\alpha^{2}}.
\]
\end{lem}
\begin{proof}
The case of $\gamma<0$ is infeasible as it would require $\tr(W)<0$
with $W\succeq0$. For $\gamma\ge0,$ we begin by relaxing the norm
constraint into an inequality, as in $\|\e\|\cdot\|\X y-\vec(W)\|\le1$.
Solving the resulting convex optimization over $y$ with a fixed $W$
yields 
\begin{align}
y^{\star} & =\X^{\dagger}[\tau\cdot\e+\vec(W)], & \|\e\|\|\X y^{\star}-\vec(W)\| & =1, & \tau & \ge0.\label{eq:trust_region}
\end{align}
The problem is feasible if and only if $\|\e\|\|(I-\X\X^{\dagger})\vec(W)\|\le1$.
Whenever feasible, the relaxation is tight, and equality is attained.
The remaining problem over $W$ reads (after some rearranging):
\begin{align*}
 & \underset{W\succeq0}{\text{minimize}}\quad\e^{T}(I-\X\X^{\dagger})\vec(W)\quad\text{subject to }\quad\langle\X^{T}\X,W\rangle=2\beta\cdot\gamma
\end{align*}
and this reduces to the following using Lemma~\ref{lem:svd} and
Lemma~\ref{lem:edecomp}:
\begin{align*}
 & \underset{W\succeq0}{\text{minimize}}\quad-(\|e\|\sin\theta)\langle v_{2}v_{2}^{T},W\rangle\quad\text{subject to }\quad2\|x\|^{2}\langle I+2v_{1}v_{1}^{T},W\rangle=2\beta\cdot\gamma
\end{align*}
with minimizer
\begin{align*}
\vec(W^{\star}) & =\frac{2\beta\cdot\gamma}{2\|x\|^{2}}(v_{2}\otimes v_{2})=\frac{\gamma}{\|e\|}\hat{e}_{2}.
\end{align*}
Clearly, we have feasibility $\|e\|\|(I-\X\X^{\dagger})\vec(W^{\star})\|\le1$
if and only if $\gamma\le1$. Substituting this particular $W^{\star}$
into (\ref{eq:trust_region}) yields
\begin{align*}
\X y^{\star}-\vec(W^{\star}) & =\frac{\sqrt{1-\gamma^{2}}}{\|e\|}\hat{e}_{1}+\frac{\gamma}{\|e\|}\hat{e}_{2}.
\end{align*}
Substituting (\ref{eq:edecomp}) yields
\begin{align*}
\e^{T}[\X y^{\star}-\vec(W^{\star})] & =\sqrt{1-\gamma^{2}}\cos\theta+\gamma\sin\theta
\end{align*}
as desired.
\end{proof}

\subsection{\label{subsec:proof_delta_lb}Proof of Theorem~\ref{thm:delta_lb}}

We consider the condition number optimization problem from~\citet{zhang2018much}:
\begin{equation}
\eta(x,z)\quad\equiv\quad\max_{\eta,\H}\left\{ \eta\quad:\quad\X^{T}\H\e=0,\quad2\mat(\H\e)+\X^{T}\H\X\succeq0,\quad\eta I\preceq\H\preceq I\right\} .\label{eq:prim_1}
\end{equation}
Its optimal value satisfies the following identity with respect to
our original LMI in (\ref{eq:tau_prob}):
\begin{equation}
\delta(x,z)=\LMI(x,z)=\frac{1-\eta(x,z)}{1+\eta(x,z)}=1-\frac{2}{1+1/\eta(x,z)}.\label{eq:delta_eta}
\end{equation}
The latter equality shows that $\delta(x,z)$ is a decreasing function
of $\eta(x,z)$. This allows us to lower-bound $\delta(x,z)$ by upper-bounding
$\eta(x,z)$.

Next, we relax (\ref{eq:prim_1}) to the following problem
\begin{equation}
\eta_{\ub}(x,z)\quad\equiv\quad\max_{\eta,\H}\left\{ \eta\quad:\quad\X^{T}\H\e=0,\quad2\mat(\H\e)+\X^{T}\X\succeq0,\quad\eta I\preceq\H\preceq I\right\} .\label{eq:prim_2}
\end{equation}
This yields an upper-bound $\eta_{\ub}(x,z)$ on $\eta(x,z)$ because
$\H\succeq I$ implies $\X^{T}\X\succeq\X^{T}\H\X$. Problem (\ref{eq:prim_2})
has Lagrangian dual (we write $v=\vec(V)$ to simplify notation)
\begin{align}
\underset{y,U_{1},U_{2},V=\mat(v)}{\text{ minimize }}\quad & \tr(U_{2})+\langle\X^{T}\X,V\rangle\label{eq:dual_1}\\
\text{subject to }\quad & (\X y-v)\e^{T}+\e(\X y-v)^{T}=U_{1}-U_{2}\nonumber \\
 & \tr(U_{1})=1,\quad U_{1},U_{2},V\succeq0.\nonumber 
\end{align}
The dual is strictly feasible (for sufficiently small $\epsilon$,
set $y=0,$ $V=\epsilon I,$ $U_{1}=\eta I-\epsilon W,$ and $U_{2}=\eta\cdot I_{n^{2}}+\epsilon W$
with suitable $\eta$ and $W$), so Slater's condition is satisfied,
strong duality holds, and the objectives coincide. We will implicitly
solve the primal problem (\ref{eq:prim_2}) by solving the dual problem
(\ref{eq:dual_1}).

In the case that $x=0$, problem (\ref{eq:dual_1}) yields a trivial
solution $y=0,$ $V=zz^{T}/(2\|z\|^{4}),$ $U_{1}=(z\otimes z)(z\otimes z)^{T}/\|z\|^{4},$
and $U_{2}=0$ with objective value $\eta_{\ub}(0,z)=0$.

In the case that $x\ne0$, we define $\alpha=(\|z\|\sin\phi)/\|\e\|$
and $\beta=\|x\|^{2}/\|\e\|>0$ and make a number of reductions on
the dual problem (\ref{eq:dual_1}). First, we use Lemma~\ref{lem:dual_prob}
to optimize over $U_{1}$ and $U_{2}$ and the length of $y$ to yield
\begin{equation}
\underset{y,V=\mat(v)\succeq0}{\text{ minimize }}\quad\frac{\tr([M]_{-})+\langle\X^{T}\X,V\rangle}{\tr([M]_{+})}\quad\text{ where }\quad M=(\X y-v)\e^{T}+\e(\X y-v)^{T}.\label{eq:dual_2}
\end{equation}
Here, we have divided the objective by the constraint $\tr([M]_{+})=1$
noting that the problem is homogenous over $y$ and $V$. Substituting
explicit expressions for the eigenvalues of $M$ in Lemma~\ref{lem:eig_bound}
yields 
\begin{equation}
\underset{y,V=\mat(v)\succeq0}{\text{ minimize }}\quad\frac{\langle\X^{T}\X,V\rangle+\|\e\|\|\X y-v\|(1-\cos\theta)}{\|\e\|\|\X y-v\|(1+\cos\theta)}\quad\text{where}\quad\cos\theta=\frac{\e^{T}(\X y-v)}{\|\e\|\|\X y-v\|}.\label{eq:dual_3}
\end{equation}
This is a multi-objective optimization over two competing trade-offs:
minimizing $\langle\X^{T}\X,V\rangle$ and maximizing $\cos\theta$.
To balance these two considerations, we parameterize over a fixed
$\gamma=\langle\X^{T}\X,V\rangle/2\beta$ and use Lemma~\ref{lem:keylb}
to maximize $\cos\theta$. The resulting univariate optimization reads
\begin{align}
\eta_{\ub}(x,z)\quad=\quad\min_{0\le\gamma\le\alpha} & \Psi(\gamma)\equiv\frac{2\beta\cdot\gamma+[1-\psi(\gamma)]}{1+\psi(\gamma)}\label{eq:final_prob}
\end{align}
where the function $\psi(\gamma)=\gamma\alpha+\sqrt{1-\gamma^{2}}\sqrt{1-\alpha^{2}}$
defined on Lemma~\ref{lem:keylb} takes on the role of the best choice
of $\cos\theta$. Here, one limit $\gamma=0$ sets $\langle\X^{T}\X,V\rangle=0,$
while the other $\gamma=\alpha$ sets $\cos\theta=1$. We cannot have
$\gamma<0$ because $V\succeq0$. Any choice of $\gamma>\alpha$ will
be strictly dominated by $\gamma=\alpha$, because $\gamma=\alpha$
already maximizes $\cos\theta$.

The univariate problem (\ref{eq:final_prob}) is \emph{quasiconvex}.
This follows from the \emph{concavity} of $\psi(\gamma)$ over this
range:
\begin{align*}
\psi'(\gamma) & =\alpha-\gamma\frac{\sqrt{1-\alpha^{2}}}{\sqrt{1-\gamma^{2}}}, & \psi''(\gamma) & =-\frac{\sqrt{1-\alpha^{2}}}{\sqrt{1-\gamma^{2}}}-\gamma^{2}\frac{\sqrt{1-\alpha^{2}}}{(1-\gamma^{2})^{3/2}}.
\end{align*}
Hence, the level sets of $\Psi(\gamma)\ge0$ are convex:
\[
\Psi(\gamma)\le c\quad\iff\quad2\beta\cdot\gamma+(1-c)\le(1+c)\psi(\gamma).
\]
We will proceed to solve the problem in closed-form and obtain
\begin{alignat}{2}
\Psi(\gamma^{\star})=\min_{0\le\gamma\le\alpha}\Psi(\gamma)\quad=\quad & \frac{1-\sqrt{1-\alpha^{2}}}{1+\sqrt{1-\alpha^{2}}} & \qquad\text{if }\beta & \ge\frac{\alpha}{1+\sqrt{1-\alpha^{2}}},\label{eq:delta_lb_a-1}\\
 & \frac{\beta(\beta-\alpha)}{\beta\alpha-1} & \text{if }\beta & \le\frac{\alpha}{1+\sqrt{1-\alpha^{2}}}.\label{eq:delta_lb_b-1}
\end{alignat}
Substituting $\Psi(\gamma^{\star})=\eta_{\ub}(x,z)$ into $\eta_{\ub}(x,z)\ge\eta(x,z)$
and using $\eta(x,z)$ to lower-bound $\delta(x,z)$ via (\ref{eq:delta_eta})
completes the proof of the lemma. (Note that setting $x=0$ sets $\beta=0$
and yields $\Psi(\gamma^{\star})=\eta_{\ub}(0,z)=0$ as desired.)

First, we verify whether the optimal solution $\gamma^{\star}$ lies
on the boundary of the search interval $[0,\alpha]$, that is $\gamma^{\star}\in\{0,\alpha\}$.
Taking derivatives yields 
\[
\Psi'(\gamma)=\frac{[2\beta-\psi'(\gamma)](1+\psi(\gamma))-\psi'(\gamma)[2\beta\cdot\gamma+1-\psi(\gamma)]}{(1+\psi(\gamma))^{2}}.
\]
For $\gamma=0$ to be a stationary point, we require $\Psi'(0)\ge0$,
and hence
\begin{align*}
[2\beta-\psi'(0)](1+\psi(0)) & \ge\psi'(0)[2\beta\cdot0+1-\psi(0)]\\
\iff\qquad\beta & \ge\frac{\alpha}{1+\sqrt{1-\alpha^{2}}}.
\end{align*}
In this case, we have $\Psi(0)=(1-\sqrt{1-\alpha^{2}})/(1+\sqrt{1-\alpha^{2}})$,
which is the expression in (\ref{eq:delta_lb_a-1}). The choice $\gamma=\alpha$
cannot be stationary, because $\Psi'(\alpha)\le0$ would imply
\begin{align*}
[2\beta-\psi'(\alpha)](1+\psi(\alpha))-\psi'(\alpha)[2\beta\cdot\alpha+1-\psi(\alpha)] & \le0,\\
\iff\qquad2\beta(1+1) & \le0,
\end{align*}
which is impossible as we have $\beta=\|x\|^{2}/\|e\|>0$ by hypothesis. 

Otherwise, the optimal solution $\gamma^{\star}$ lies in the interior
of the search interval $[0,\alpha]$, that is $\gamma^{\star}\in(0,\alpha)$.
In this case, we simply relax the bound constraints on $\gamma$ and
solve the unconstrained problem as a linear fractional conic program
\begin{align*}
\min_{|\gamma|\le1}\Psi(\gamma)=\min_{\|\xi\|\le1} & \left\{ \frac{1+(c-d)^{T}\xi}{1+d^{T}\xi}\right\} 
\end{align*}
where
\[
c=\begin{bmatrix}2\beta\\
0
\end{bmatrix},\qquad d=\begin{bmatrix}\alpha\\
\sqrt{1-\alpha^{2}}
\end{bmatrix},\qquad\xi=\begin{bmatrix}\gamma\\
\sqrt{1-\gamma^{2}}
\end{bmatrix}.
\]
(Note that the relaxation $\|\xi\|\le1$ is always tight, because
the linear fractional objective is always monotonous with respect
to scaling of $\xi$.) Defining $q=\xi/(1+d^{T}\xi)$ and $q_{0}=1/(1+d^{T}\xi)\ge0$
rewrites this as the second-order cone program
\[
\Psi(\gamma^{\star})=\underset{\|q\|\le q_{0}}{\min}\left\{ \begin{bmatrix}c-d\\
1
\end{bmatrix}^{T}\begin{bmatrix}q\\
q_{0}
\end{bmatrix}:\begin{bmatrix}d\\
1
\end{bmatrix}^{T}\begin{bmatrix}q\\
q_{0}
\end{bmatrix}=1\right\} 
\]
that admits a strictly feasible point $q=0$ and $q_{0}=1$. Accordingly,
the Lagrangian dual has zero duality gap:
\[
\Psi(\gamma^{\star})=\max_{\lambda}\{\lambda:\|c-(1+\lambda)d\|\le(1-\lambda)\}.
\]
If the maximum $\lambda^{\star}$ exists, then it must attain the
inequality, as in
\begin{align*}
\|c-(1+\lambda^{\star})d\|^{2} & =(1-\lambda^{\star})^{2}.
\end{align*}
We can simply solve this quadratic equation
\[
c^{T}c-2c^{T}d(1+\lambda^{\star})+(1+\lambda^{\star})^{2}d^{T}d=4-4(\lambda^{\star}+1)+(\lambda^{\star}+1)^{2}
\]
for the optimal $\lambda^{\star}=\Psi(\gamma^{\star})$. Noting that
$d^{T}d=1$, we actually have just a single root
\begin{align*}
1+\lambda^{\star} & =\frac{(c^{T}c-4)}{2(c^{T}d-2)}=\frac{\beta^{2}-1}{\alpha\beta-1}=1+\frac{\beta(\beta-\alpha)}{\alpha\beta-1},
\end{align*}
and this yields the expression (\ref{eq:delta_lb_b-1}).

\section{\label{sec:Numerical-Results}Numerical Results}

An important advantage of our formulation is that $\delta(X,Z)$ can
be evaluated numerically in cases where an exact closed-form solution
does not exist (or is too difficult to obtain). In this section, we
augment our analysis with a numerical study. In the rank $r=1$ case,
we exhaustively evaluate $\delta(x,z)$ over its two degrees of freedom
to gain insight on its behavior, and also to quantify the conservatism
of the lower-bound in Theorem~\ref{thm:delta_lb}. In the rank $r\ge1$
case, we sample $\delta(X,Z)$ uniformly at random over $X$ and $Z,$
in order to understand its distribution and hypothesize on higher-rank
versions of our recovery guarantees.

In our experiments, we implement Algorithm~\ref{alg:efficient_alg}
in MATLAB. We parse the LMI problem using YALMIP~\citep{lofberg2004yalmip}
and solve it using MOSEK~\citep{andersen2000mosek}. All algorithms
parameters (e.g. accuracy, iterations, etc.) are left at their default
values. 
\begin{figure}
\hfill{}\includegraphics[width=0.4\textwidth]{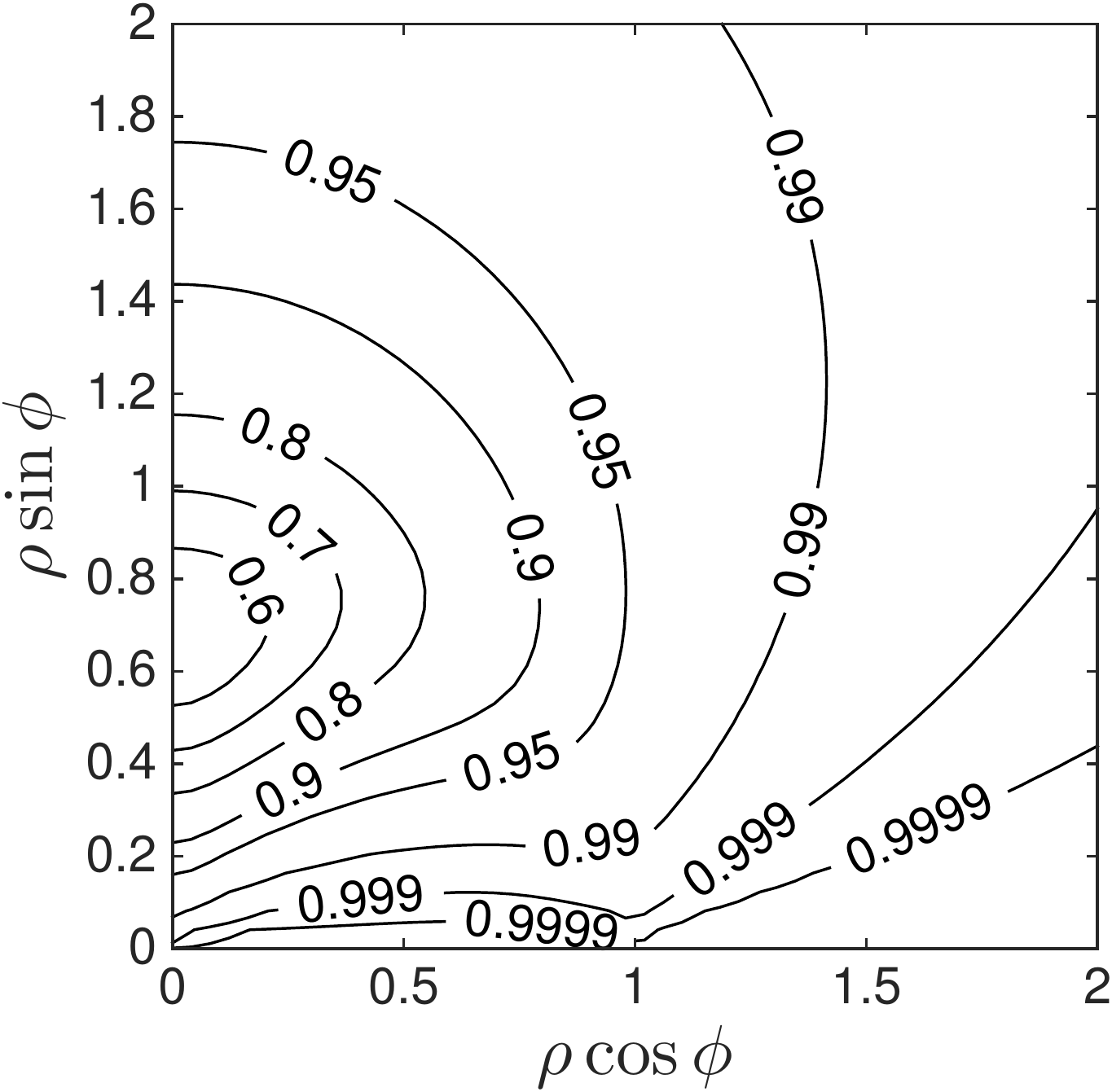}\hfill{}\includegraphics[width=0.4\textwidth]{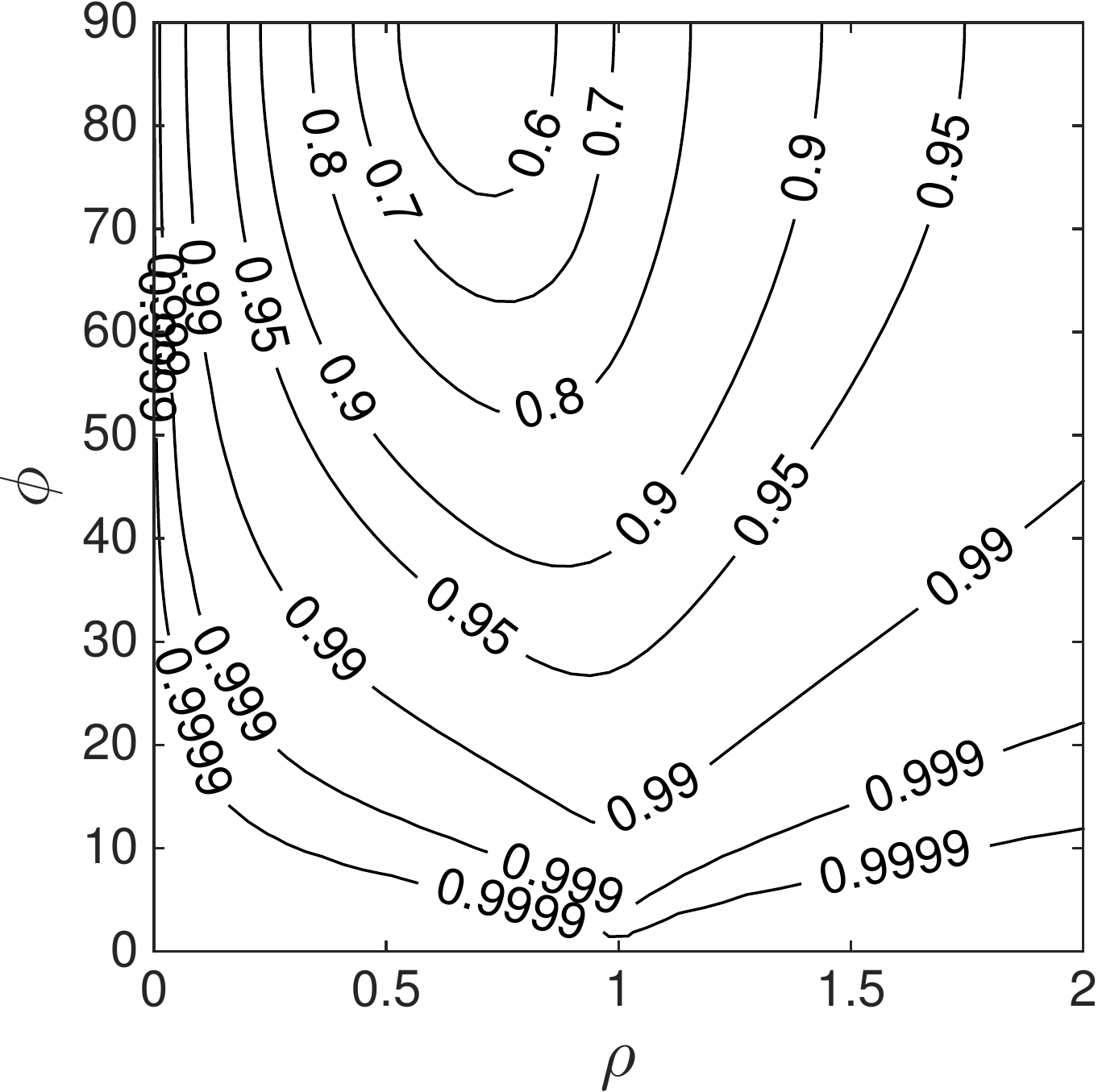}\hfill{}

\vspace{0.2in}

\hfill{}\includegraphics[width=0.4\textwidth]{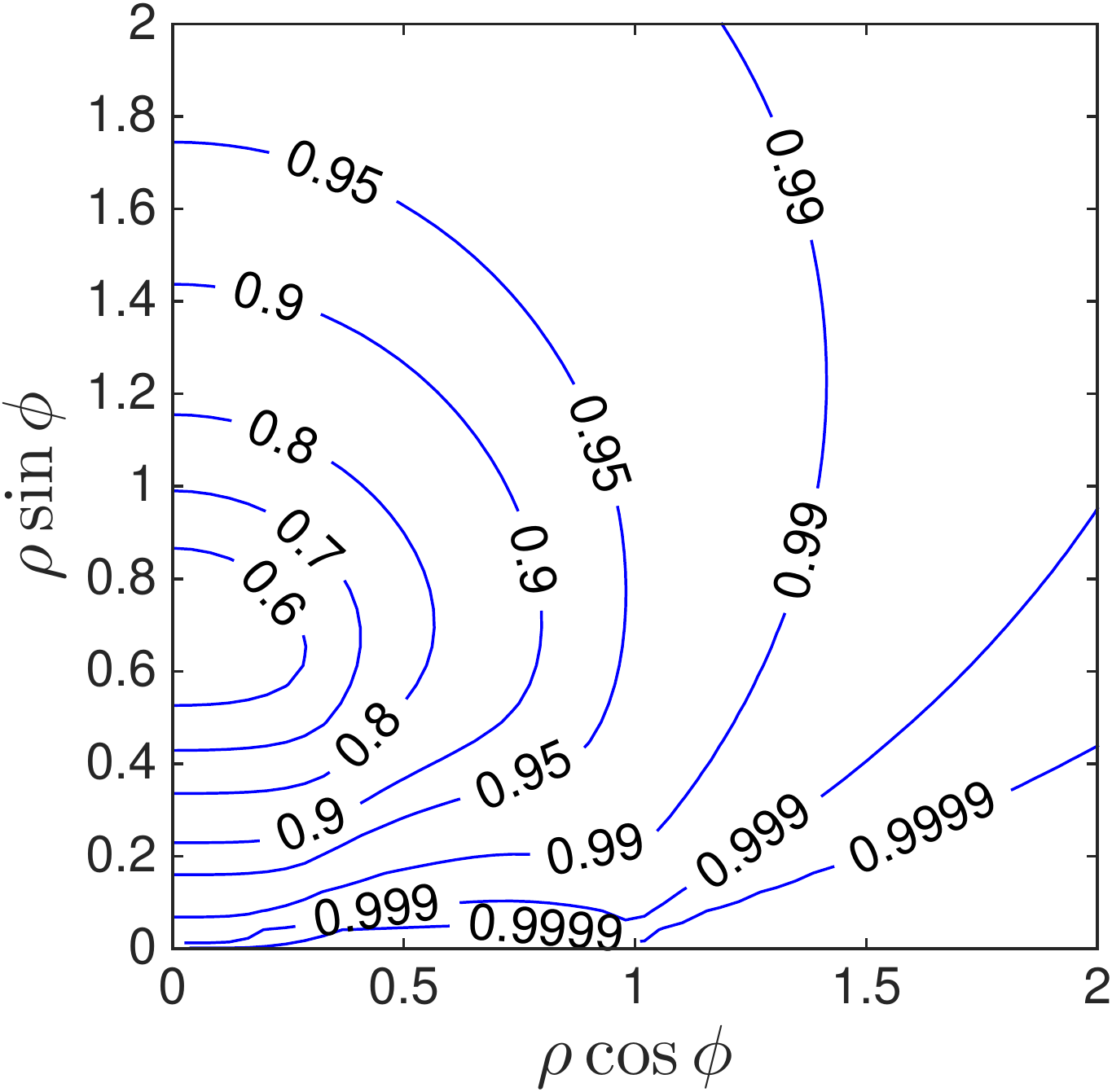}\hfill{}\includegraphics[width=0.4\textwidth]{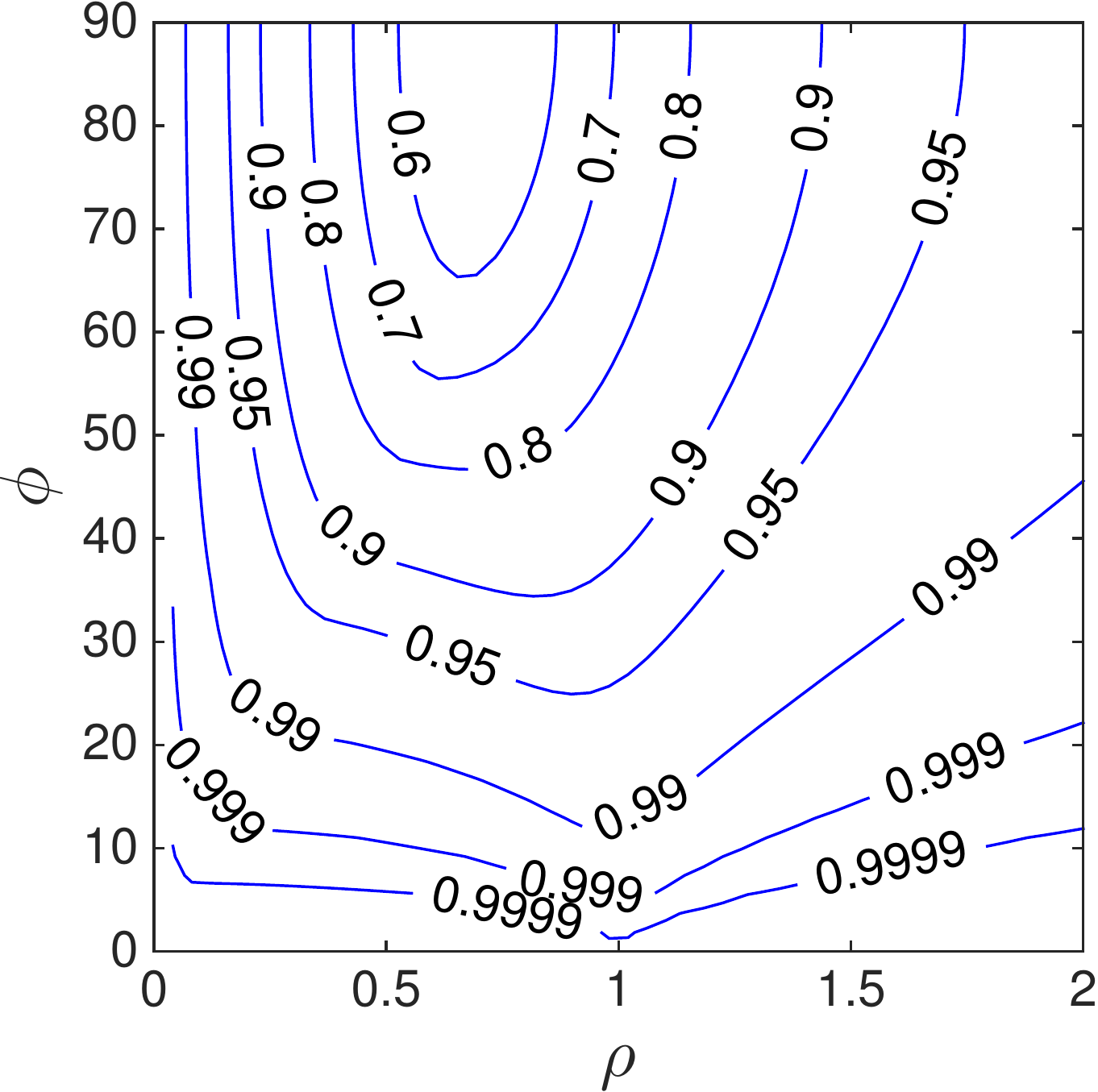}\hfill{}

\vspace{0.2in}

\hfill{}\includegraphics[width=0.4\textwidth]{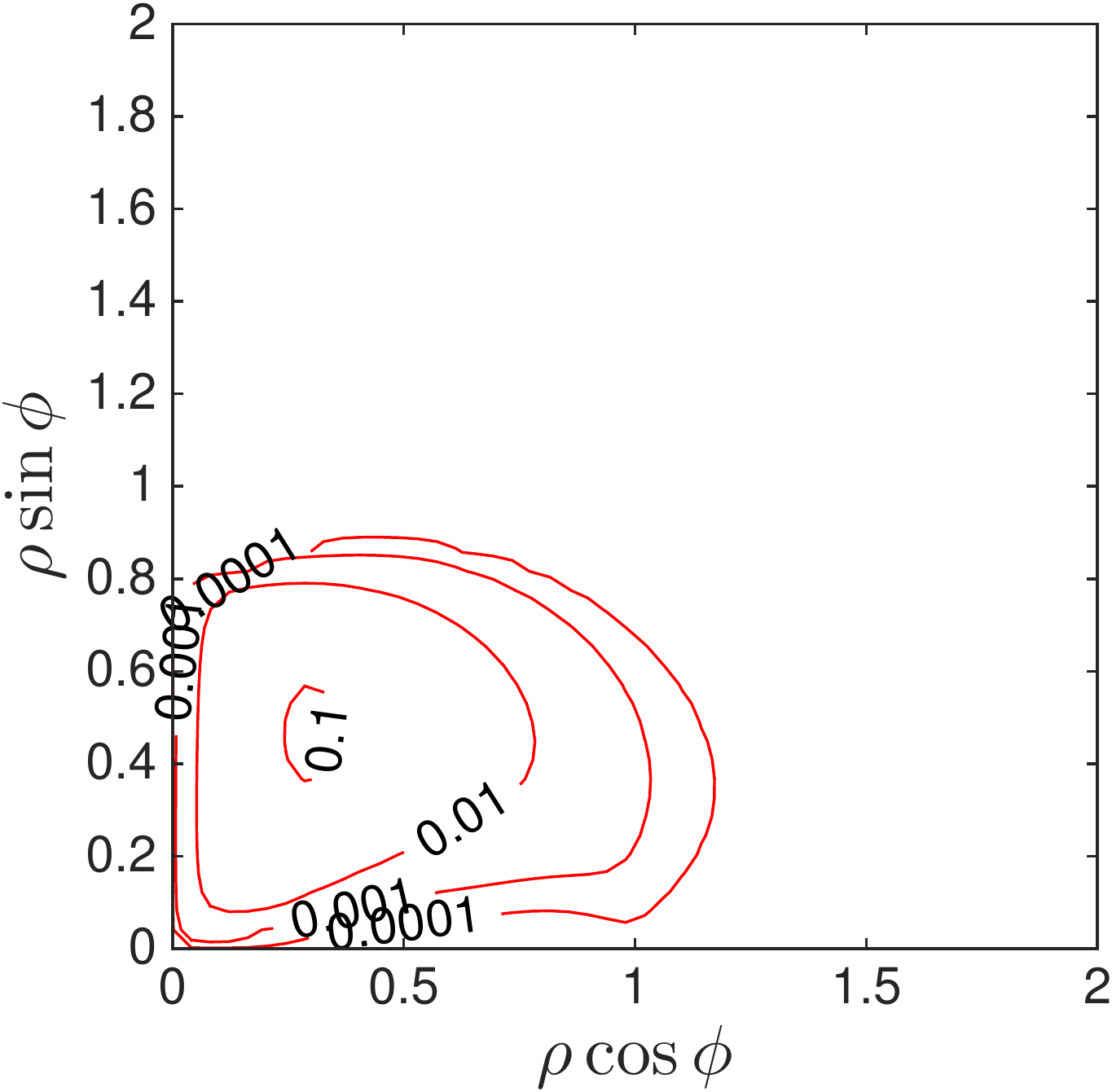}\hfill{}\includegraphics[width=0.4\textwidth]{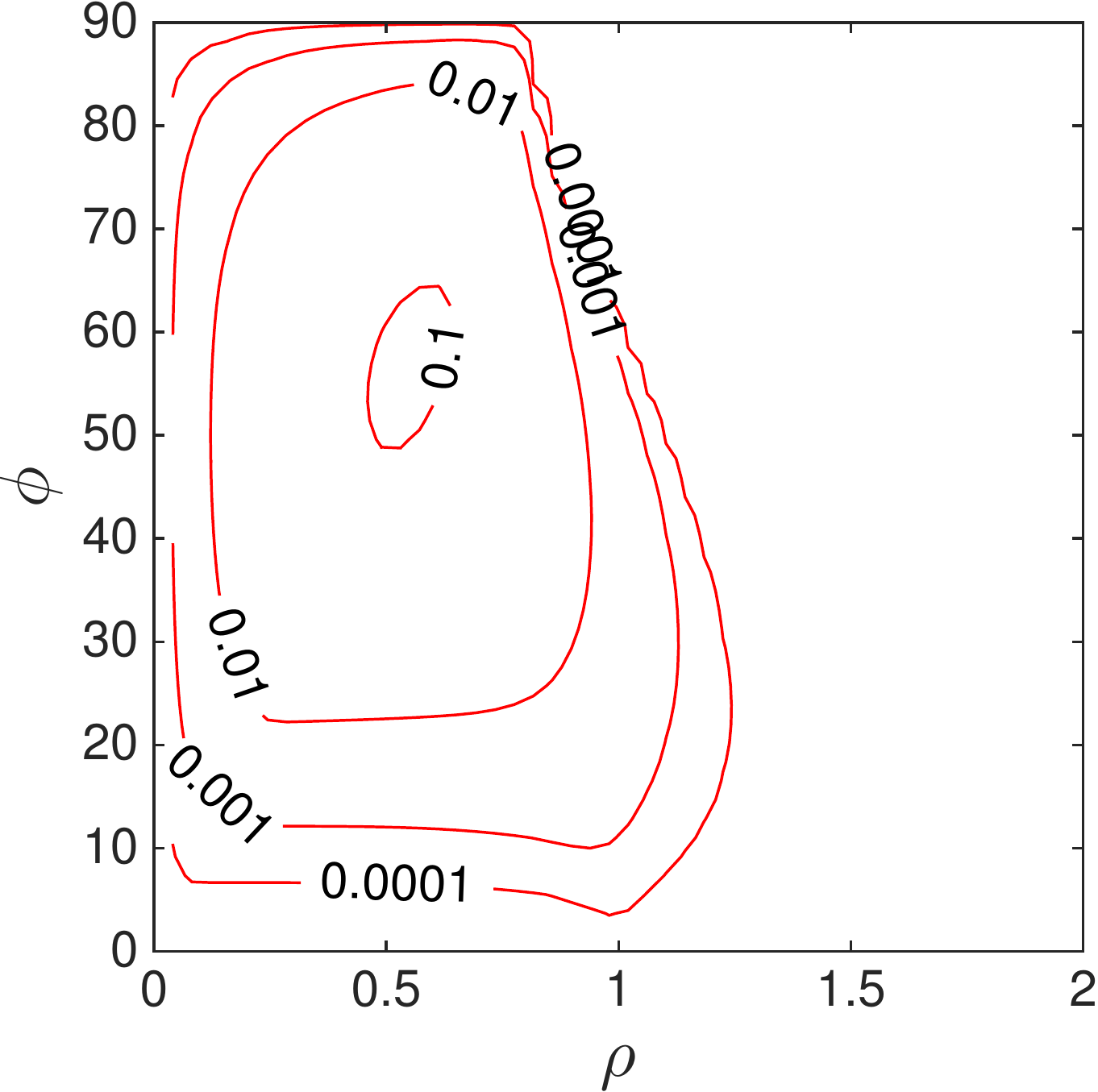}\hfill{}

\caption{\label{fig:plot2}The function $\delta(x,z)$ and its lower bound
$\delta_{\protect\lb}(x,z)$ visualized with respect to the length
ratio $\rho=\|x\|/\|z\|$ and the incidence angle $\phi=\arccos(x^{T}z/\|x\|\|z\|)$:
\textbf{(top)} the function $\delta(x,z)$; \textbf{(middle)} the
lower-bound $\delta_{\protect\lb}(x,z)$; \textbf{(bottom)} the error
$\delta(x,z)-\delta_{\protect\lb}(x,z)$; \textbf{(left)} rectangular
coordinates; \textbf{(right)} polar coordinates.}
\end{figure}

\subsection{Visualizing $\delta(x,z)$ and $\delta_{\protect\lb}(x,z)$ for rank
$r=1$}

Using a suitable orthogonal projector $P$, we can reduce the function
$\delta(x,z)$ down to two underlying degrees of freedom: the length
ratio $\rho=\|x\|/\|z\|$ and the incidence angle $\phi=\arccos(x^{T}z/\|x\|\|z\|)$.
First, without loss of generality, we assume that the ground truth
$M^{\star}=zz^{T}$ has unit norm $\|M^{\star}\|_{F}=1$. (Otherwise,
we can suitably rescale all arguments below.) Then, the following
projector $P$ satisfies $PP^{T}x=x$ and $PP^{T}z=z$ with 
\begin{align}
P & =\begin{bmatrix}z & {\displaystyle \frac{(I-zz^{T})x}{\|(I-zz^{T})x\|}}\end{bmatrix}, & P^{T}x & =\begin{bmatrix}\rho\cos\phi\\
\rho\sin\phi
\end{bmatrix}, & P^{T}z & =\begin{bmatrix}1\\
0
\end{bmatrix}.
\end{align}
Applying this particular $P$ to Lemma~\ref{lem:reduc} yields the
following
\begin{equation}
\delta(x,z)=\delta(P^{T}x,P^{T}z)=\delta\left(\begin{bmatrix}\rho\cos\phi\\
\rho\sin\phi
\end{bmatrix},\begin{bmatrix}1\\
0
\end{bmatrix}\right).\label{eq:delta_2var}
\end{equation}
In fact, this two-variable function is symmetric over its four rectangular
quadrants
\begin{equation}
\delta\left(\begin{bmatrix}\rho\cos\phi\\
\rho\sin\phi
\end{bmatrix},\begin{bmatrix}1\\
0
\end{bmatrix}\right)=\delta\left(\begin{bmatrix}\pm\rho\cos\phi\\
\pm\rho\sin\phi
\end{bmatrix},\begin{bmatrix}1\\
0
\end{bmatrix}\right)\label{eq:delta_2varsym}
\end{equation}
because either $\pm z$ corresponds to the same ground truth, and
because the second column of $P$ can point in either $\pm(I-zz^{T})x$.

Accordingly, we can use (\ref{eq:delta_2var}) and (\ref{eq:delta_2varsym})
to visualize $\delta(x,z)$ as a two-dimensional graph, either in
rectangular coordinates over $(\rho\cos\phi,\rho\sin\phi)\in[0,\rho_{\max}]^{2}$,
or in polar coordinates over $(\rho,\phi)\in[0,\rho_{\max}]\times[0,\pi/2]$.
Moreover, we can plot our closed-form lower-bound $\delta_{\lb}(x,z)$
on the same axes, in order to quantify its conservatism $\delta(x,z)-\delta_{\lb}(x,z)$.

The top row of Figure~\ref{fig:plot2} plots $\delta(x,z)$ in rectangular
and polar coordinates. The plot shows $\delta(x,z)$ as a smooth function
with a single basin at $\rho=1/\sqrt{2}$ and $\phi=90^{\circ}$.
Outside of a narrow region with $1/2\le\rho\le1$ and $\phi\ge45^{\circ}$,
we have $\delta(x,z)\ge0.9$. For smaller RIP constants, spurious
local minima must appear in a narrow region\textemdash they cannot
occur arbitrarily anywhere. Excluding this region\textemdash as in
our local guarantee in Theorem~\ref{thm:local}\textemdash allows
much larger RIP constants $\delta$ to be accommodated.

The middle and bottom rows of Figure~\ref{fig:plot2} plots $\delta_{\lb}(x,z)$
and $\delta(x,z)-\delta_{\lb}(x,z)$ in rectangular and polar coordinates.
The two functions match within $0.01$ for either $\rho\ge1$ or $\phi\le30^{\circ}$,
and fully concur in the asymptotic limits $\rho\to\{0,+\infty\}$
and $\phi\to\{0^{\circ},90^{\circ}\}$. The greatest error of around
0.1 occurs at $\rho\approx0.5$ and $\phi\approx55^{\circ}$. We conclude
that $\delta_{\lb}(x,z)$ is a high quality approximation for $\delta(x,z)$.

\subsection{Distribution of $\delta(X,Z)$ for rank $r\ge1$}

In the high-rank case, a simple characterization of $\delta(X,Z)$
is much more elusive. Given a fixed rank-$r$ ground truth $M^{\star}$,
let its corresponding eigendecomposition be written as $M^{\star}=V\Lambda V^{T}$
where $V\in\R^{n\times r}$ is orthogonal and $\Lambda$ is diagonal.
By setting $Z=V\Lambda^{1/2}$ and suitably selecting an orthogonal
projector $P$, it is always possible to satisfy
\begin{equation}
\delta(X,Z)=\delta(P^{T}X,P^{T}Z)=\delta\left(\begin{bmatrix}\hat{X}_{1}\\
\hat{X}_{2}
\end{bmatrix},\begin{bmatrix}\Lambda^{1/2}\\
0
\end{bmatrix}\right),\label{eq:delta_2var-1}
\end{equation}
where $\hat{X}_{1},\hat{X}_{2}\in\R^{r\times r}$. While (\ref{eq:delta_2var-1})
bares superficial similarities to (\ref{eq:delta_2var}), the equation
now contains at least $2r^{2}+r-1$ degrees of freedom. Even $r=2$
results in 9 degrees of freedom, which is too many to visualize.

Instead, we sample $\delta(X,Z)$ uniformly at random over its underlying
degrees of freedom. Specifically, we select all elements in $X\in\R^{n\times r}$
and only the diagonal elements of $Z\in\R^{n\times r}$ independently
and identically distributed from the standard Gaussian, as in $X_{i,j},Z_{i,i}\sim\text{Gaussian}(0,1)$.
We then use Algorithm~\ref{alg:efficient_alg} to evaluate $\delta(X,Z)$.
\begin{figure}
\hfill{}\includegraphics[width=0.4\textwidth]{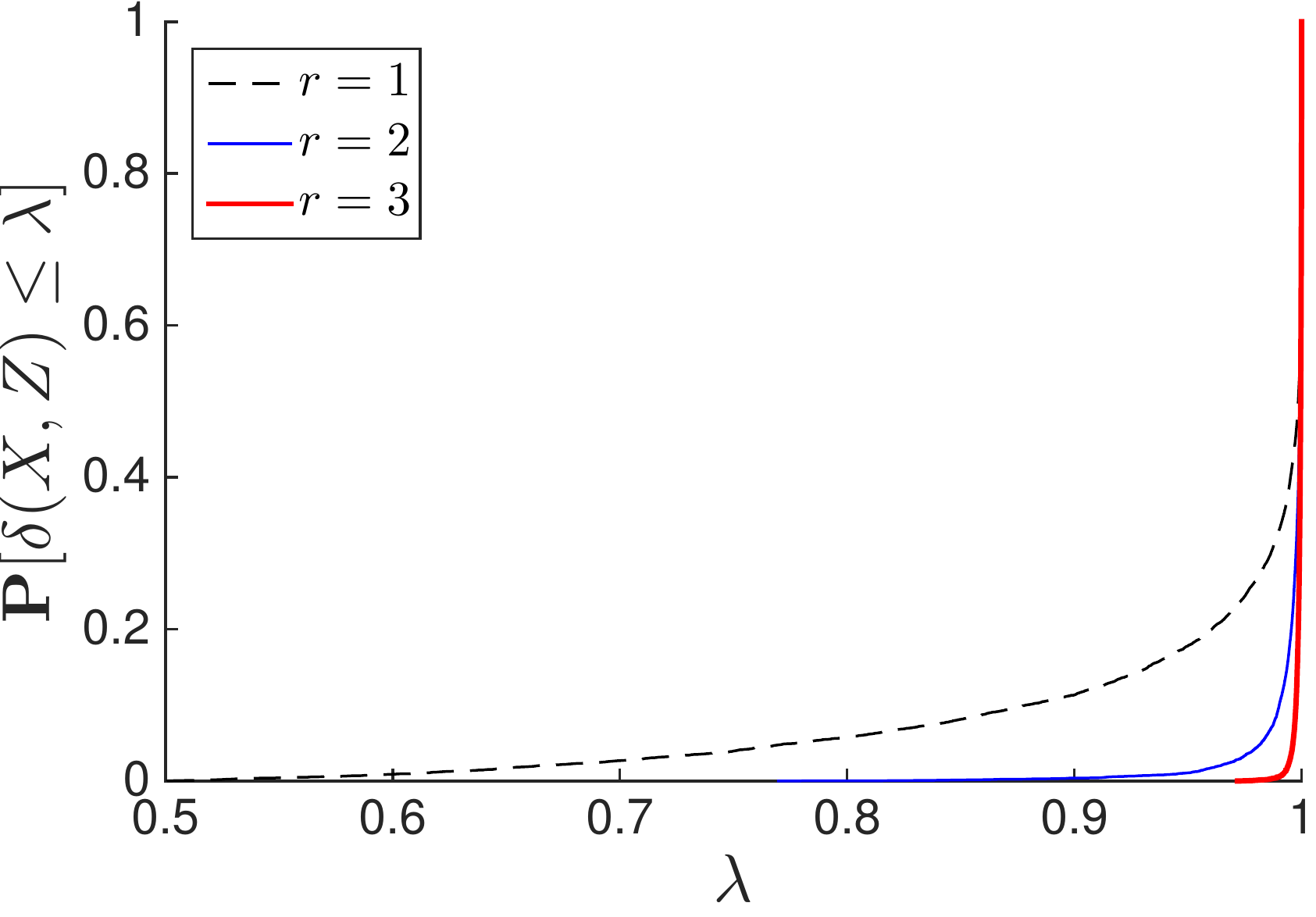}\hfill{}\includegraphics[width=0.4\textwidth]{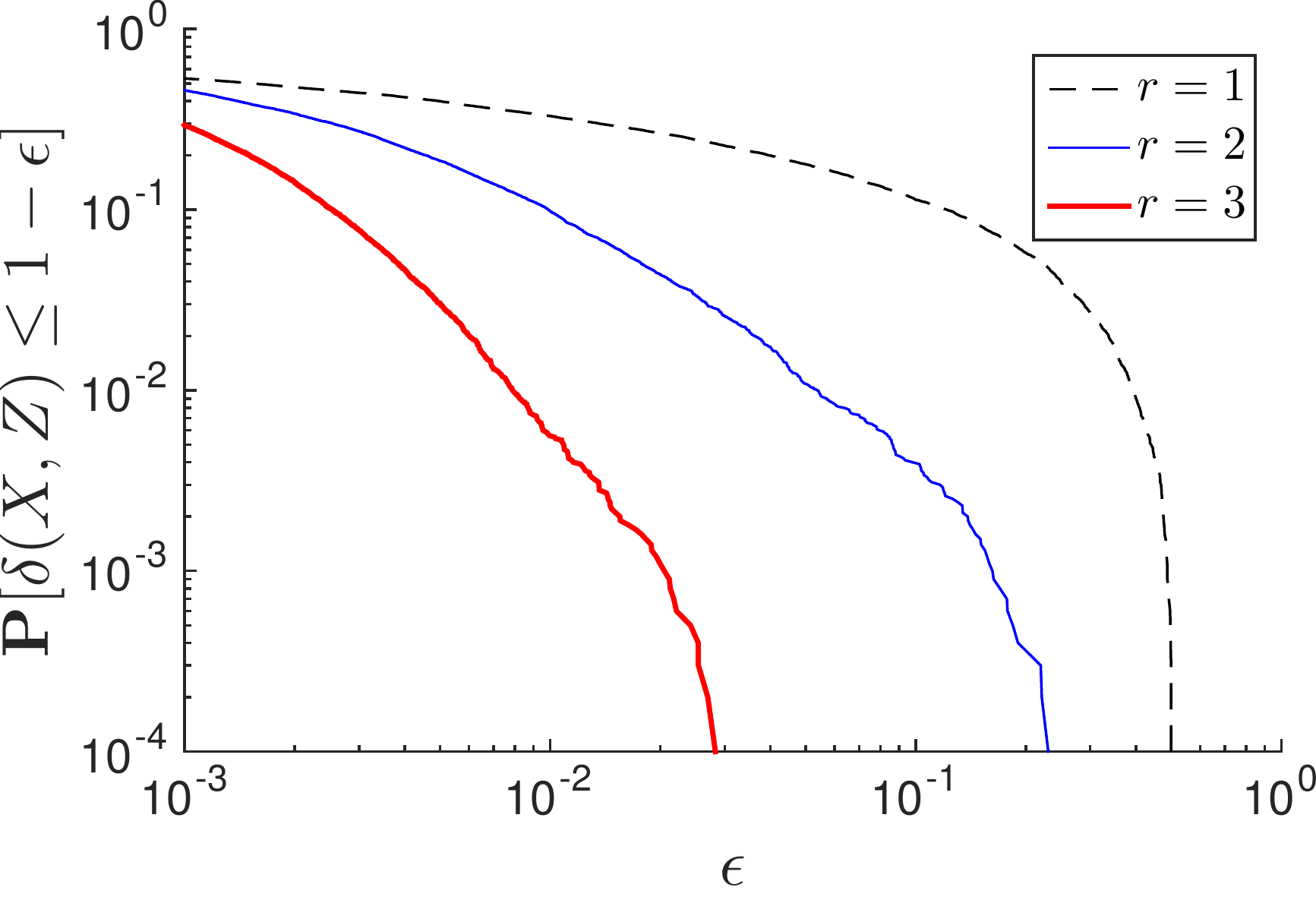}\hfill{}

\caption{\label{fig:ecdf}Empirical cumulative distribution of $\delta(X,Z)$
over $N=10^{4}$ samples of $X,Z\in\protect\R^{n\times r}$ where
$X_{i,j},Z_{i,i}\sim\text{Gaussian}(0,1)$: \textbf{(left)} linear
plot of $\mathbf{P}[\delta(X,Z)\le\lambda]$ over $\lambda\in[1/2,1]$;
\textbf{(right)} logarithmic plot of $\mathbf{P}[\delta(X,Z)\le1-\epsilon]$
over the tail $\epsilon\in[10^{-3},10^{0}]$.}
\end{figure}

Figure~\ref{fig:ecdf} plots the empirical cumulative distributions
for $r\in\{1,2,3\}$ from $N=10^{4}$ samples. We see that each increase
in rank $r$ results in a sizable reduction in the distribution tail.
The rank $r=1$ trials yielded $\delta(x,z)$ arbitrarily close to
the minimum value of $1/2$, but the rank $r=2$ trials were only
able to find $\delta(X,Z)\approx0.8$. The rank $r=3$ trials were
even more closely concentrated about one, with the minimum at $\delta(X,Z)\approx0.97$.
These results suggest that higher rank problems are generically easier
to solve, because larger RIP constants are sufficient to prevent the
points from being spurious local minima. They also suggest that $\delta(X,Z)\ge1/2$
over \emph{all} rank $r\ge1$, though this is not guaranteed, because
``bad'' choices of $X,Z$ can always exist on a lower-dimensional
zero-measure set.

\section{Conclusions}

The low-rank matrix recovery problem is known to contain \emph{no
spurious local minima} under a restricted isometry property (RIP)
with a sufficiently small RIP constant $\delta$. In this paper, we
introduce a proof technique capable of establishing RIP thresholds
that are both \emph{necessary} and \emph{sufficient} for exact recovery.
Specifically, we define $\delta(X,Z)$ as the \emph{smallest} RIP
constant associated with a counterexample with \emph{fixed} ground
truth $M^{\star}=ZZ^{T}$ and \emph{fixed} spurious point $X$, and
define $\delta^{\star}=\min_{X,Z}\delta(X,Z)$ as the smallest RIP
constant over all counterexamples. Then, $\delta$-RIP low-rank matrix
recovery contains no spurious local minima if and only if $\delta<\delta^{\star}$. 

Our key insight is to show that $\delta(X,Z)$ has an \emph{exact}
convex reformulation. In the rank-1 case, the resulting problem is
sufficiently simple that it can be relaxed and solved in closed-form.
Using this closed-form bound, we prove that $\delta<1/2$ is both
necessary and sufficient for exact recovery from any arbitrary initial
point. For larger RIP constants $\delta\ge1/2$, we show that an initial
point $x_{0}$ satisfying $f(x_{0})\le(1-\delta)^{2}f(0)$ is enough
to guarantee exact recovery using a descent algorithm. It is important
to emphasize, however, that these sharp results are derived specifically
for the rank-1 case. 

\section*{Acknowledgements}

We are grateful to Salar Fattahi for a meticulous reading and detailed
comments, and to Salar Fattahi and Cédric Josz for fruitful discussions.
We thank two anonymous reviewers for helpful comments and for pointing
out typos. This work was supported by grants from ONR, AFOSR, ARO,
and NSF.

\appendix

\section{\label{sec:reduc}Detailed proof of Lemma~\ref{lem:reduc}}

Given $X,Z\in\R^{n\times r},$ define $\e\in\R^{n^{2}}$ and $\X\in\R^{n^{2}\times nr}$
to satisfy the following with respect to $X$ and $Z$
\begin{align}
\e & =\vec(XX^{T}-ZZ^{T}), & \X\vec(Y) & =\vec(XY^{T}+YX^{T})\qquad\forall Y\in\R^{n\times r},\label{eq:eXdef_restate}
\end{align}
Let $P\in\R^{n\times d}$ with $d\le n$ satisfy 
\begin{align*}
P^{T}P & =I_{d}, & PP^{T}X & =X, & PP^{T}Z & =Z
\end{align*}
and define $\P=P\otimes P$ and the projections $\hat{X}=P^{T}X$
and $\hat{Z}=P^{T}Z$. Define $\hat{\e}\in\R^{d^{2}}$ and $\hat{\X}\in\R^{d\times dr}$
to satisfy (\ref{eq:eXdef_restate}) with $X,Z$ replaced by $\hat{X},\hat{Z}$.

Our goal is to show that 
\begin{align*}
\delta & =\hat{\delta}, & \H & =\P\hat{\H}\P^{T}+(I-\P\P^{T}),
\end{align*}
satisfy the primal feasibility equations\begin{subequations}\label{eq:pf}
\begin{gather}
\X^{T}\H\e=0,\label{eq:pf1}\\
2[I_{r}\otimes\mat(\H\e)]+\X^{T}\H\e\succeq0,\label{eq:pf2}\\
(1-\delta)I\preceq\H\preceq(1+\delta)I,\label{eq:pf3}
\end{gather}
\end{subequations}and that 
\begin{align*}
y & =(I_{r}\otimes P)\hat{y}, & U_{1} & =\P\hat{U}_{1}\P^{T}, & U_{2} & =\P\hat{U}_{2}\P^{T}, & V & =(I_{r}\otimes P)\hat{V}(I_{r}\otimes P)^{T}
\end{align*}
satisfy the dual feasibility equations \begin{subequations}\label{eq:df}
\begin{gather}
\sum_{j=1}^{r}(\X y-\vec(V_{j,j}))\e^{T}+\e(\X y-\vec(V_{j,j}))^{T}-\X V\X^{T}=U_{1}-U_{2},\label{eq:df1}\\
\tr(U_{1}+U_{2})=1,\label{eq:df2}\\
V\succeq0,\quad U_{1}\succeq0,\quad U_{2}\succeq0,\label{eq:df3}
\end{gather}
\end{subequations}under the hypothesis that $(\hat{\delta},\hat{\H})$
and $(\hat{y},\hat{U}_{1},\hat{U}_{2},\hat{V})$ satisfy (\ref{eq:pf})
and (\ref{eq:df}) with $\e,\X$ replaced by $\hat{\e},\hat{\X}$.

We can immediately verify (\ref{eq:pf3}), (\ref{eq:df2}), and (\ref{eq:df3})
using the orthogonality of $\P$. To verify the remaining equations,
we will use the following identities. 
\begin{claim}
We have
\begin{align*}
\e & =\P\hat{\e}, & \X(I_{r}\otimes P) & =\P\hat{\X} & \P^{T}\X & =\hat{\X}(I_{r}\otimes P)^{T}.
\end{align*}
\end{claim}
\begin{proof}
For all $Y\in\R^{n\times r}$ and $\hat{Y}\in\R^{d\times r},$ we
have
\begin{gather*}
\e=\vec(XX^{T}-ZZ^{T})=\vec[P(\hat{X}\hat{X}^{T}-\hat{Z}\hat{Z}^{T})P^{T}]=(P\otimes P)\hat{\e},\\
\X(I_{r}\otimes P)\vec(\hat{Y})=\X\vec(P\hat{Y})=\vec[P(\hat{X}\hat{Y}^{T}+\hat{Y}\hat{X}^{T})P^{T}]=\P\hat{\X}\vec(\hat{Y}),\\
\P^{T}\X\vec(Y)=\vec[(P^{T}X)(P^{T}Y)^{T}+(P^{T}Y)(P^{T}X)^{T}]=\hat{\X}\vec(P^{T}Y)=\hat{\X}(I_{r}\otimes P)^{T}\vec(Y).
\end{gather*}
\end{proof}
Now, we have (\ref{eq:df1}) from
\begin{gather*}
\X y-\vec(V_{j,j})=\X(I_{r}\otimes P)\hat{y}-\P\vec(\hat{V}_{j,j})=\P(\hat{\X}\hat{y}-\vec(\hat{V}_{j,j})),\\
\X V\X^{T}=\X(I_{r}\otimes P)\hat{V}(I_{r}\otimes P)^{T}\X=\P(\hat{\X}\hat{V}\hat{\X}^{T})\P^{T}.
\end{gather*}
To prove (\ref{eq:pf1}), we use
\[
\X^{T}\H\e=\X^{T}\H(\P\hat{\e})=\X^{T}(\P\hat{\H})\hat{\e}=(I_{r}\otimes P)\hat{\X}^{T}\hat{\H}\hat{\e}.
\]
Lastly, to prove (\ref{eq:pf2}), we define
\[
\mathbf{S}=2\cdot[I_{r}\otimes\mat(\H e)]+\X^{T}\H\X
\]
and $P_{\perp}$ as the orthogonal complement of $P$. Then, observe
that
\begin{align*}
I_{r}\otimes\mat(\H e) & =I_{r}\otimes(P\,\mat(\hat{\H}\hat{e})\,P^{T})=(I_{r}\otimes P)(I_{r}\otimes\mat(\hat{\H}\hat{e}))(I_{r}\otimes P)^{T},
\end{align*}
and that
\[
\X^{T}\H\X(I_{r}\otimes P)=\X^{T}\H(\P\hat{\X})=\X^{T}(\P\hat{\H})\hat{\X}=(I_{r}\otimes P)\hat{\X}^{T}\hat{\H}\hat{\X}.
\]
Hence, we have
\begin{gather*}
(I_{r}\otimes P)^{T}\mathbf{S}(I_{r}\otimes P)=2\cdot[I_{r}\otimes\mat(\hat{\H}\hat{e})]+\hat{\X}\hat{\H}\hat{\X}\succeq0,\\
(I_{r}\otimes P_{\perp})^{T}\mathbf{S}(I_{r}\otimes P_{\perp})=(I_{r}\otimes P_{\perp})^{T}\X^{T}\H\X(I_{r}\otimes P_{\perp})\succeq0,\\
(I_{r}\otimes P_{\perp})^{T}\mathbf{S}(I_{r}\otimes P)=0,
\end{gather*}
and this shows that $\mathbf{S}\succeq0$ as desired.

\section{\label{sec:tightness}Detailed proof of Lemma~\ref{lem:tightness}}

Given $X,Z\in\R^{n\times r},$ let $P=\orth([X,Z])$ and $\P=P\otimes P$.
Our goal is to show that 
\begin{align*}
\delta & =\hat{\delta}, & \H & =\P\hat{\H}\P^{T}
\end{align*}
satisfy the primal feasibility equations\begin{subequations}\label{eq:pf-1}
\begin{gather}
\X^{T}\H\e=0,\label{eq:pf1-1}\\
2[I_{r}\otimes\mat(\H\e)]+\X^{T}\H\e\succeq0,\label{eq:pf2-1}\\
(1-\delta)I\preceq\P^{T}\H\P\preceq(1+\delta)I,\label{eq:pf3-1}
\end{gather}
\end{subequations}and that
\begin{align*}
y & =(I_{r}\otimes P)\hat{y}, & U_{1} & =\hat{U}_{1}, & U_{2} & =\hat{U}_{2}, & V & =(I_{r}\otimes P)\hat{V}(I_{r}\otimes P)^{T}
\end{align*}
satisfy the dual feasibility equations\begin{subequations}\label{eq:df-1}
\begin{gather}
\sum_{j=1}^{r}(\X y-\vec(V_{j,j}))\e^{T}+\e(\X y-\vec(V_{j,j}))^{T}-\X V\X^{T}=\P(U_{1}-U_{2})\P^{T},\label{eq:df1-1}\\
\tr[\P(U_{1}+U_{2})\P^{T}]=1,\label{eq:df2-1}\\
V\succeq0,\quad U_{1}\succeq0,\quad U_{2}\succeq0,\label{eq:df3-1}
\end{gather}
\end{subequations}under the hypothesis that $(\hat{\delta},\hat{\H})$
and $(\hat{y},\hat{U}_{1},\hat{U}_{2},\hat{V})$ satisfy (\ref{eq:pf})
and (\ref{eq:df}) with $\e,\X$ replaced by $\hat{\e},\hat{\X}$.
The exact steps for verifying (\ref{eq:pf-1}) and (\ref{eq:df-1})
are identical to the proof of Lemma~\ref{lem:reduc}, and are omitted
for brevity.

\bibliography{euclid}

\end{document}